\documentclass{article}
\usepackage{flo_preprint,times}
\usepackage{amsmath,amssymb,amsthm}
\usepackage{booktabs,array,tabularx,multirow,graphicx,xcolor}
\usepackage{enumitem,caption}
\usepackage{hyperref}
\usepackage{url}
\usepackage{float}
\floatstyle{ruled}
\newfloat{algorithm}{tbp}{loa}
\floatname{algorithm}{Algorithm}
\floatstyle{plain}
\usepackage{wrapfig}
\makeatletter 
\newcommand\FloatBarrier{\par\begingroup\let\@elt\relax\edef\@tempa{\@botlist\@deferlist\@dbldeferlist}\ifx\@tempa\@empty\endgroup\else\endgroup\clearpage\fi}
\makeatother
\hypersetup{
  hidelinks,
  pdftitle={More Domains, No Gains? FLO: Fast--Slow Reciprocal Learning for Multi-Domain Graph Pre-training},
  pdfauthor={Yupeng Tong, Lianze Shan, Jitao Zhao, Di Jin, Dongxiao He},
  pdfsubject={Multi-Domain Graph Pre-training}
}
\newtheoremstyle{floplain}{\topsep}{0pt}{\itshape}{}{\bfseries}{.}{ }{}
\newtheoremstyle{flodefinition}{\topsep}{0pt}{\normalfont}{}{\bfseries}{.}{ }{}
\theoremstyle{floplain}
\newtheorem{proposition}{Proposition}[section]
\newtheorem{lemma}[proposition]{Lemma}
\newtheorem{theorem}[proposition]{Theorem}

\theoremstyle{flodefinition}
\newtheorem{definition}[proposition]{Definition}
\newenvironment{restate}[1]{\par\addvspace{\medskipamount}\noindent\textbf{#1 (restated).}\ \itshape\ignorespaces}{\par\addvspace{\medskipamount}}
\newcommand{\NA}{\textemdash}
\title{More Domains, No Gains? FLO: Fast--Slow Reciprocal Learning for Multi-Domain Graph Pre-training}
\author{\textbf{Yupeng Tong}$^{1\dagger}$ \quad \textbf{Lianze Shan}$^{1\dagger}$ \quad \textbf{Jitao Zhao}$^{1}$ \quad \textbf{Di Jin}$^{1}$ \quad \textbf{Dongxiao He}$^{1*}$ \\ $^{1}$School of Computer Science and Technology, Tianjin University, Tianjin, China \\ \texttt{\{cualb, shanlz2119, zjtao, jindi, hedongxiao\}@tju.edu.cn}}
\begin{document}
\maketitle
\begingroup
\renewcommand{\thefootnote}{\fnsymbol{footnote}}
\footnotetext[2]{These authors contributed equally to this work.}
\footnotetext[1]{Corresponding author: Dongxiao He (\href{mailto:hedongxiao@tju.edu.cn}{hedongxiao@tju.edu.cn}).}
\endgroup
\begin{abstract}
Inspired by the great success of foundation models in the development of artificial general intelligence, Graph Foundation Models (GFMs) have recently attracted significant attention and have the potential to advance graph intelligence. Exhibiting such intelligence requires effectively learning knowledge from large-scale, multi-domain graph data during pre-training. Increasing domain diversity during pre-training is expected to yield richer and more generalizable knowledge. Most existing methods pursue this goal by designing domain alignment strategies that help a single graph model learn from multiple domains. However, we observe that existing methods struggle to benefit from increasing domain diversity. Through empirical and theoretical analysis, we reveal two key limitations. Firstly, pre-training in most GFMs is confined to a single message-passing paradigm. Since propagation flattens relations among neighboring nodes, the graph model can hardly perceive the knowledge beyond propagation, namely the attribute-predictable component of these relations. Secondly, domain alignment strategies often rely on randomly initialized parameters that are jointly updated with the graph model using immediate loss feedback, leaving early knowledge acquisition vulnerable to cross-domain interference. To address these issues, we propose FLO, a Fast--sLow reciprOcal approach for multi-domain graph pre-training. Specifically, a propagation-free attribute branch predicts the graph model's relations without message passing and returns these predictions to refine the model. We further stabilize knowledge acquisition by pairing a fast process that adapts to the current domain with a slow process that integrates knowledge across domains. On 12 datasets, FLO benefits from increasing domain diversity and achieves the best average performance in cross-domain and cross-task transfer.
\end{abstract}

\section{Introduction}
\label{sec:introduction}
In recent years, foundation models \citep{bommasani2021opportunities} have achieved great success in advancing the development of artificial general intelligence, fundamentally reshaping various fields of machine learning. Inspired by this progress, Graph Foundation Models (GFMs) have emerged as a research frontier, attracting significant attention for their potential to advance graph intelligence \citep{mao2024position,liu2025graph,li2024zerog,he2025unigraph,zhu2024graphcontrol,zhao2025fully,sun2025riemanngfm,xia2026anygraph,sun2026multi}. By modeling complex relational structures inherent in real-world systems, GFMs hold the promise of establishing a unified paradigm for graph machine learning \citep{liu2024one,pmlr-v267-wang25eq,bevilacqua2025holographic}. However, exhibiting such highly generalizable intelligence imposes a critical challenge: these models must be capable of effectively extracting and learning universal knowledge from large-scale, multi-domain graph data during pre-training \citep{hu2020strategies,qiu2020gcc,you2020graph}.

Increasing domain diversity during pre-training is expected to yield richer and more generalizable knowledge. To achieve this, most existing methods\footnote{More related work is discussed in Appendix~\ref{app:relatedwork}.} focus on carefully designing domain alignment strategies to help a single graph model reconcile the heterogeneity across diverse domains. Typically, these methods rely on unifying feature spaces or introducing learnable alignment tokens to bridge semantic and structural gaps across domains. For the former, these approaches project heterogeneous node attributes into a shared, domain-agnostic space \citep{shan2026leda,eliasof2026bridging,pmlr-v267-shen25p,sun2025handling,shan2026mug}. For instance, FUG \citep{zhao2024fug} designs a dimension encoder to capture transferable distributions among attributes, whereas TIG \citep{zhao2026towards} proposes transfer-invariant metadata, namely feature-dimension relations shared across graphs, and mines it to guide the pre-training process. For the latter, these methods explicitly reconcile complex cross-domain differences by augmenting the Graph Neural Network (GNN) encoder with auxiliary learnable parameters \citep{zhao2024all,yuan2025much,zhuo2026modality}. For example, MDGPT \citep{yu2024text} employs learnable domain tokens as semantic anchors across various datasets, while SAMGPT \citep{yu2025samgpt} introduces structure tokens to align distinct topological characteristics. With these strategies, a graph model is expected to gain more knowledge as more source domains are added.

\begin{wrapfigure}[14]{r}{0.53\linewidth}
\centering
\includegraphics[width=\linewidth]{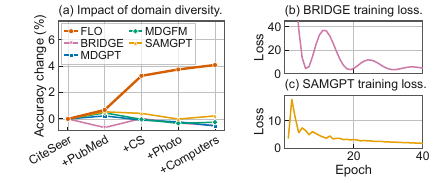}
\caption{Motivation experiments with Cora as the target graph. (a) Accuracy change as source domains are added. (b)--(c) Pre-training loss of two baselines.}
\label{fig:motivation}
\end{wrapfigure}
However, we observe a counter-intuitive phenomenon: as shown in Figure~\ref{fig:motivation}(a), existing methods struggle to benefit from increasing domain diversity. This suggests that domain alignment is a prerequisite rather than a guarantee of knowledge acquisition. Through detailed empirical and theoretical analysis, we identify two limitations that remain after domain alignment. \textbf{Firstly,} after domain alignment, the graph model still forms node relations mainly through message passing \citep{wang2026beyond,yang2026graph}. Propagation smooths the representations of neighboring nodes \citep{li2018deeper,cai2020note} and flattens their relations. The model can thus hardly perceive the knowledge beyond propagation, namely the attribute-predictable component of these relations. \textbf{Secondly,} the aforementioned alignment components (e.g., domain or structure tokens) often rely on randomly initialized parameters that are jointly updated with the graph model using immediate loss feedback. Given the severe data heterogeneity, this immediate joint optimization paradigm leaves early knowledge acquisition vulnerable to cross-domain interference, ultimately destabilizing the pre-training process (Figure~\ref{fig:motivation}(b)--(c)).

To address the above limitations, we propose FLO, a multi-domain graph pre-training approach with Fast--sLow reciprOcal learning. Specifically, we predict the relations of the graph model without message passing and return the predictions as feedback, so that the graph model acquires knowledge beyond propagation. We further pair a fast process that adapts to the current domain with a slow process that integrates knowledge across domains \citep{mcclelland1995there,polyak1992acceleration,zhang2019lookahead,cha2021swad,arani2022learning,busbridge2023scale,morales-brotons2024exponential}, and take the exchanged relations from the slow process. This protects early knowledge acquisition from cross-domain interference. After pre-training, we deploy only the graph model on unseen graphs, so FLO adds no parameters at inference. As shown in Figure~\ref{fig:motivation}(a), the accuracy of FLO increases steadily as source domains are added. Our contributions are summarized as follows:
\begin{itemize}[leftmargin=1.3em,itemsep=1pt,topsep=3pt]
\item We find that existing methods struggle to benefit from increasing domain diversity. We reveal two limitations behind this: message passing alone can hardly perceive knowledge beyond propagation, and randomly initialized alignment parameters updated by immediate loss feedback leave early knowledge acquisition vulnerable to cross-domain interference.
\item We propose FLO, which refines the node relations induced by message passing with feedback grounded in node attributes alone and takes this feedback from a slow process that integrates knowledge across domains. We prove that the returned relations lead the graph model toward the knowledge beyond propagation without domain-specific parameters, while the slow process filters cross-domain interference from the feedback.
\item Experiments on 12 datasets show that FLO achieves the best average performance in cross-domain and cross-task transfer and benefits from increasing domain diversity.
\end{itemize}

\section{Preliminaries}
\label{sec:preliminaries}
Unless otherwise specified, we use calligraphic letters for sets, bold uppercase letters for matrices, bold lowercase letters for vectors, and blackboard bold letters for spaces.

\textbf{Graph data.} A graph $G=(\mathcal V,\mathcal T)$ has node set $\mathcal V$ with $n$ nodes, edge set $\mathcal T$, adjacency matrix $\mathbf A$, and feature matrix $\mathbf X\in\mathbb R^{n\times d}$ with rows $\mathbf x_i$. In this paper, we consider undirected graphs. Let $\widetilde{\mathbf A}=\mathbf A+\mathbf I$, $\widetilde{\mathbf D}_{ii}=\sum_j\widetilde{\mathbf A}_{ij}$, and $\widehat{\mathbf A}=\widetilde{\mathbf D}^{-1/2}\widetilde{\mathbf A}\widetilde{\mathbf D}^{-1/2}$.

\textbf{Graph neural networks.} A Graph Convolutional Network (GCN) layer \citep{kipf2017semisupervised} computes:
\begin{equation}
\mathbf H^{(\ell+1)}=\sigma\bigl(\widehat{\mathbf A}\mathbf H^{(\ell)}\mathbf W^{(\ell)}\bigr),
\label{eq:gcn}
\end{equation}
where $\mathbf W^{(\ell)}$ is a learnable weight matrix and $\sigma$ is a nonlinear activation. It aggregates the representations of neighboring nodes through message passing, and this propagation smooths them \citep{li2018deeper}. The Dirichlet energy $\mathcal E(\mathbf Y)=\operatorname{tr}(\mathbf Y^\top(\mathbf I-\widehat{\mathbf A})\mathbf Y)$ measures how much a signal $\mathbf Y$ varies across edges, and propagation lowers it (Section~\ref{flo:sec:theory}).

\textbf{Multi-domain pre-training.} We are given $S$ unlabeled source graphs $\{G_s\}_{s=1}^{S}$ with adjacency matrices $\mathbf A_s$ and raw feature matrices $\mathbf X_s^{\rm raw}$, whose feature dimensions and feature semantics differ across graphs. We treat each graph dataset as a domain. The goal is to pre-train a graph encoder on the source graphs without any labels, such that it produces node representations for a target graph from an unseen domain without being updated on it. To align the feature dimensions of different graphs, we map each raw feature matrix to a common dimension $d$:
\begin{equation}
\mathbf X_s=\operatorname{RowNorm}\!\left(\operatorname{Pad}_d(\widetilde{\mathbf X}_s\mathbf V_s)\right),\qquad
\widetilde{\mathbf X}_s=\operatorname{RowNorm}(\mathbf X_s^{\rm raw}),
\label{eq:frontend}
\end{equation}
where $\mathbf V_s$ contains up to $d$ leading right singular vectors of $\widetilde{\mathbf X}_s$, $\operatorname{Pad}_d$ appends zeros when needed, and $\operatorname{RowNorm}$ normalizes nonzero rows to unit length. The target graph undergoes the same preprocessing, which uses no labels.

\section{Method}
\label{flo:sec:method}
\label{sec:method}
As illustrated in Figure~\ref{fig:framework}, FLO addresses two problems that remain after domain alignment: what relational knowledge a shared graph encoder acquires (Section~\ref{flo:sec:parr}), and how this knowledge stays stable as source domains alternate (Section~\ref{flo:sec:fska}). We then present the training objective and transfer procedure in Section~\ref{flo:sec:objective}, followed by the theoretical analysis in Section~\ref{flo:sec:theory}.

\subsection{Propagation--Attribute Relation Refinement}
\label{flo:sec:parr}
Section~\ref{flo:sec:theory} shows how propagation can flatten relations among nearby nodes. To address this limitation, we design Propagation--Attribute Relation Refinement (PARR), which pairs the graph encoder with a propagation-free attribute branch. The key is the information asymmetry between the two branches rather than the second branch itself: the attribute branch is deliberately denied access to propagation. This constraint makes its optimal relation prediction the attribute-predictable component $\overline P$ (Definition~\ref{flo:def:beyond}), turning reciprocal relation learning into a conditional prediction-and-refinement process rather than ordinary agreement between peer encoders \citep{zhang2018deep,feng2022freekd}.

Specifically, the graph branch uses a two-layer GCN as the graph encoder $f_G$. The attribute branch $f_M$ uses the identity matrix in place of $\widehat{\mathbf A}$ and is implemented as a two-layer MLP \citep{zhang2022graphless,tian2023learning,zhou2025tined} with the same layer widths. The two branches have separate parameters $\theta_G$ and $\theta_M$, each shared across source graphs. For source graph $s$, we sample two independent elementwise masks $\mathbf M_G,\mathbf M_M\in\{0,1\}^{n_s\times d}$ with keep probability $p$ and compute:
\begin{equation}
\mathbf Z_G=f_G(\mathbf X_s\odot\mathbf M_G,\mathbf A_s;\theta_G),
\qquad
\mathbf Z_M=f_M(\mathbf X_s\odot\mathbf M_M;\theta_M).
\label{flo:eq:encoding}
\end{equation}
Inspired by masked graph autoencoders \citep{hou2022graphmae}, each branch $a\in\{G,M\}$ recovers the direction of the unmasked features through a linear decoder $D_a$:
\begin{equation}
\mathcal L_{\mathrm{rec}}^a
=\frac{\gamma}{n_s}\sum_{i=1}^{n_s}
\left\|\operatorname{RowNorm}\!\left(D_a(\mathbf z_i^a)\right)-\mathbf x_i\right\|_2^2,
\label{flo:eq:recon}
\end{equation}
where $\mathbf z_i^a$ and $\mathbf x_i$ are the $i$-th rows of $\mathbf Z_a$ and $\mathbf X_s$, and $\gamma>0$ is a scaling coefficient. Since every nonzero row of $\mathbf X_s$ has unit norm by Eq.~\eqref{eq:frontend}, the decoder predicts the direction of each node feature. Reconstruction keeps each representation tied to node content while the branch learns relations.

The two branches need not share embedding coordinates, so they exchange relation distributions instead of embeddings. At each step, we sample a set $\mathcal I$ of anchor nodes and, for every anchor $i$, a candidate set $\mathcal C_i$ shared by both branches and their slow copies (Section~\ref{flo:sec:fska}). The candidates are the anchor itself, the nodes visited by short random walks from $i$, and nodes sampled uniformly from the graph (Appendix~\ref{app:implementation}). Let $\mathbf Z$ be node representations with rows $\mathbf z_i$. Following knowledge distillation \citep{hinton2015distilling,park2019relational}, in particular methods that match distributions of pairwise similarities \citep{yang2020distilling,fang2021seed,zheng2021ressl}, we define the relation distribution of anchor $i$ over its candidates as:
\begin{equation}
p_{\mathbf Z}(c\mid i)=
\frac{\exp\!\left(\cos(\mathbf z_i,\mathbf z_{c})/\tau\right)}
{\sum_{c'\in\mathcal C_i}\exp\!\left(\cos(\mathbf z_i,\mathbf z_{c'})/\tau\right)},
\qquad c\in\mathcal C_i,
\label{flo:eq:relations}
\end{equation}
where $\tau>0$ is a temperature. The relation distribution describes how similar the anchor is to each of its candidates. Relation distributions are unchanged by orthogonal transformations of the representations and positive rescaling of their rows (Appendix~\ref{flo:app:invariance}).

\begin{figure}[t]
\centering
\includegraphics[width=\linewidth]{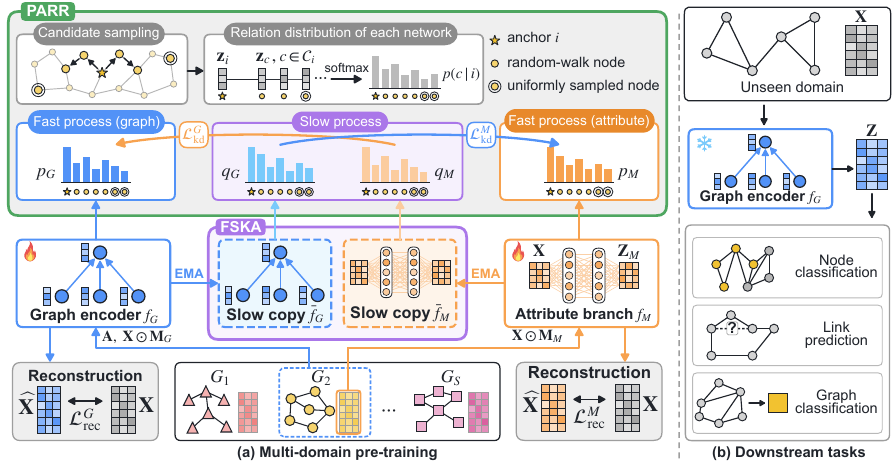}
\caption{Overview of FLO. (a) Multi-domain pre-training: source graphs are visited in turn, and the graph branch $f_G$ and the attribute branch $f_M$ each reconstruct node features. In Fast--Slow Knowledge Acquisition (FSKA), each branch keeps a slow copy updated by an Exponential Moving Average (EMA) after each step. In Propagation--Attribute Relation Refinement (PARR), each branch learns the relation distribution of the other branch's slow copy over the same candidates. (b)~Downstream tasks: only the frozen graph encoder $f_G$ is applied to graphs from unseen domains.}
\label{fig:framework}
\end{figure}
Let $p_G$ and $p_M$ denote the relation distributions of $\mathbf Z_G$ and $\mathbf Z_M$, and let $q_G$ and $q_M$ denote the relation targets supplied by FSKA (Section~\ref{flo:sec:fska}). The two directions of reciprocal relation learning are then:
\begin{equation}
\mathcal L_{\mathrm{kd}}^M=\frac{1}{|\mathcal I|}\sum_{i\in\mathcal I}
\operatorname{KL}\bigl(q_G(\cdot\mid i)\,\|\,p_M(\cdot\mid i)\bigr),
\qquad
\mathcal L_{\mathrm{kd}}^G=\frac{1}{|\mathcal I|}\sum_{i\in\mathcal I}
\operatorname{KL}\bigl(q_M(\cdot\mid i)\,\|\,p_G(\cdot\mid i)\bigr),
\label{flo:eq:bidirectional}
\end{equation}
where $\operatorname{KL}$ denotes the Kullback--Leibler divergence. Through $\mathcal L_{\mathrm{kd}}^M$, the attribute branch learns the relations formed by propagation from the attributes of the anchor and its candidates alone. Through $\mathcal L_{\mathrm{kd}}^G$, its predictions are returned to the graph encoder as relation targets. This exchange corrects the relations of the graph encoder toward those predictable from node attributes, with an error controlled by how well the attribute branch fits the relations of the graph branch (Section~\ref{flo:sec:theory}). Relation learning alone admits the trivial solution of uniform relations, which the reconstruction in Eq.~\eqref{flo:eq:recon} rules out by keeping each representation predictive of its own node features (Appendix~\ref{flo:app:recon}).

\subsection{Fast--Slow Knowledge Acquisition}
\label{flo:sec:fska}
As source domains alternate, the pre-training loss of existing methods fluctuates sharply (Figure~\ref{fig:motivation}(b)--(c)). The parameters of each branch in PARR are also updated on one source graph at a time. If these parameters also supplied the relation targets, changes induced by each source graph would be passed to the other branch through these targets. To address this issue, we introduce Fast--Slow Knowledge Acquisition (FSKA), which separates the parameters that adapt to the current source graph from those that supply the relation targets.

Inspired by momentum teachers \citep{tarvainen2017mean,he2020momentum,grill2020bootstrap}, each branch $a\in\{G,M\}$ maintains a slow copy $\overline\theta_a$ of its online parameters $\theta_a$, initialized as $\overline\theta_a=\theta_a$. The online parameters are optimized on the current source graph and form the fast process, which acquires knowledge from the current domain. The slow copy receives no gradients and forms the slow process, which integrates this knowledge across domains. After each step, it is updated by an exponential moving average:
\begin{equation}
\overline\theta_a\leftarrow\mu\,\overline\theta_a+(1-\mu)\,\theta_a,\qquad a\in\{G,M\},
\label{flo:eq:ema-update}
\end{equation}
where $\mu\in(0,1)$ is the momentum coefficient, so each online update enters the relation targets only gradually. Each slow copy uses the same masked input and candidate sets as its online branch, and supplies the relation target $q_a$ in Eq.~\eqref{flo:eq:bidirectional}. Combined with PARR, each fast branch learns from the slow copy of the other branch. PARR determines the relational target being exchanged, whereas FSKA controls the time scale at which this target enters reciprocal supervision.

\subsection{Training Objective and Transfer}
\label{flo:sec:objective}
The overall objective combines reconstruction with the two directions of relation learning:
\begin{equation}
\mathcal L=\mathcal L_{\mathrm{rec}}^G+\mathcal L_{\mathrm{rec}}^M+
\lambda\,r_e\,\bigl(\mathcal L_{\mathrm{kd}}^G+\mathcal L_{\mathrm{kd}}^M\bigr),
\qquad r_e=\min(1,e/T_w),
\label{flo:eq:objective}
\end{equation}
where $\lambda>0$ weights relation learning and $r_e$ increases its weight linearly over the first $T_w$ epochs. This warm-up lets reconstruction shape feature-dependent representations before relation learning reaches its full weight. Since the relation targets receive no gradients, each online branch $a$ is optimized only through $\mathcal L_{\mathrm{rec}}^a+\lambda r_e\mathcal L_{\mathrm{kd}}^a$. In each epoch, we visit the source graphs in turn. At each visit, we take one optimization step on both online branches and their decoders, then update the slow copies. Each epoch thus takes $S$ steps, one per source graph.

After pre-training, we discard the attribute branch, the decoders, and the slow copies. For an unseen target graph, we apply the frozen graph encoder $f_G$ to its preprocessed features and adjacency matrix, without input masks or domain-specific encoder parameters. The complete algorithm, hyperparameters, and computational cost are given in Appendix~\ref{app:implementation}.

\subsection{Theoretical Analysis}
\label{flo:sec:theory}
\label{sec:theory}
In this subsection, we analyze PARR and FSKA theoretically, measuring what a branch has learned by the information in its relation distributions. Proofs are given in Appendix~\ref{flo:app:theory}.

\textbf{Propagation erodes relational information.}
\label{flo:sec:mp-theory}
We draw a source graph $s$ with probability $\pi_s$, an anchor with its candidate slots, and the masked inputs of both branches. Let $P$ be the relation distribution of the graph branch (Eq.~\eqref{flo:eq:relations}), and let $C$ be a slot drawn from $P$. Let $\xi$ contain the features of the anchor and its candidates in both masked inputs.

\begin{definition}[Knowledge beyond propagation]
\label{flo:def:beyond}
Let $\overline P=\mathbb E[P\mid\xi]$ denote the component of the graph-branch relation distribution predictable from propagation-free attribute observations, for a fixed candidate set. We call $\overline P$ the knowledge beyond propagation.
\end{definition}
The information in the relations of the graph branch then splits as:
\begin{equation}
I(C;P)=I(C;\xi)+R,\qquad R=I(C;P\mid\xi)=\mathbb E\,\operatorname{KL}(P\|\overline P).
\label{flo:eq:decomp}
\end{equation}
The first term is carried by the knowledge beyond propagation, and $R$ is the part induced by propagation. Let $\mathcal E$ be the Dirichlet energy on graph $s$ (Section~\ref{sec:preliminaries}), and let $\zeta_s<1$ be the largest magnitude among the eigenvalues of $\widehat{\mathbf A}_s$ other than one.

\begin{theorem}[Propagation flattens relations]
\label{flo:thm:prop}
Let $\mathbf Z=\widehat{\mathbf A}_s^{K}\mathbf Y$, and let $P_i^h$ be the relations of anchor $i$ over its candidates within $h$ hops. Then the following holds:
\begin{equation}
\operatorname{KL}(P_i^h\|U)\le \beta_i\,\mathcal E(\mathbf Z)\le \beta_i\,\zeta_s^{2K}\,\mathcal E(\mathbf Y),
\label{flo:eq:flatten}
\end{equation}
where $U$ is uniform and $\beta_i$ depends only on $h$, $\tau$, the degree of $i$, and $\|\mathbf z_i\|_2$.
\end{theorem}
Since $I(C;P_i^h)\le\mathbb E\,\operatorname{KL}(P_i^h\|U)$, the relational information among nearby candidates is capped by the Dirichlet energy. By Eq.~\eqref{flo:eq:flatten}, this cap shrinks at a rate set by the spectrum of each graph, so aligning the inputs of different graphs does not change this rate.

\textbf{Reciprocal relation learning removes the propagation-induced part.}
\label{flo:sec:attribute-theory}
Since the attribute branch performs no propagation, its relation distribution $Q$ is a function of $\xi$ and contains no information induced by propagation.

\begin{theorem}[Reciprocal relation learning]
\label{flo:thm:recip}
(i) For every attribute branch $Q$, $\;\mathbb E\,\operatorname{KL}(P\|Q)=R+\mathbb E\,\operatorname{KL}(\overline P\|Q)\ge R$, with equality if and only if $Q=\overline P$.
(ii) If the relations of the graph branch move toward a returned target $Q$, that is, $P'=(1-\eta)P+\eta Q$ with $\eta\in[0,1]$, then $R'\le(1-\eta)R$. If $Q=\overline P$, then in addition $I(C';\xi)=I(C;\xi)$.
(iii) $\mathbb E\|P-\overline P\|_1^2\le2R$, and under the conditions of Appendix~\ref{flo:app:recip}, $R\le\kappa\sum_s\pi_s\sqrt{\mathcal E(\mathbf X_s)}$.
\end{theorem}
Part (i) shows that learning from the graph branch drives the attribute branch toward the knowledge beyond propagation. Part (ii) shows that moving the relations a fraction $\eta$ toward an attribute target removes at least a fraction $\eta$ of $R$, and a target equal to $\overline P$ also keeps the information predicted by attributes. Part (iii) relates the size of this correction to the Dirichlet energy of each graph, and FLO makes it with parameters shared by all graphs. On our six graphs, the propagation-induced relation change correlates with the drop without relation learning ($r=0.89$, Appendix~\ref{app:relation-change}).

\textbf{Slowly updated targets filter cross-domain interference.}
\label{flo:sec:ema-theory}
FLO visits the $S$ source graphs in turn and takes one update on each. Over a window of training, we write each update as $\delta_t=\mathbf g+\mathbf e_{s(t)}$, where $\mathbf g$ is shared by all graphs and $\sum_s\mathbf e_s=\mathbf 0$. The online parameters then follow $\mathbf g$ and oscillate with period $S$. We model the cross-domain interference of Section~\ref{sec:introduction} as this oscillation. The large early loss fluctuations in Figure~\ref{fig:motivation}(b)--(c) motivate this model.

\begin{proposition}[Slowly updated targets as a low-pass filter]
\label{flo:prop:slow}
The slowly updated parameters $\overline\theta_t=\mu\overline\theta_{t-1}+(1-\mu)\theta_t$ follow $\mathbf g$ with a lag of $\mu/(1-\mu)$ steps. In the steady state, their period-$S$ oscillation is at most $\rho$ times that of the online parameters, where $\rho$ is given by:
\begin{equation}
\rho=\frac{1-\mu}{\sqrt{(1-\mu)^2+4\mu\sin^2(\pi/S)}}.
\label{flo:eq:rho}
\end{equation}
The relation targets inherit this attenuation up to a Lipschitz constant.
\end{proposition}
With $\mu=0.99$ and $S=5$, $\rho\approx0.009$, so the oscillation shrinks by two orders of magnitude. In summary, PARR reduces the propagation-induced part $R$ of the relations and keeps $\overline P$ (Theorem~\ref{flo:thm:recip}), and FSKA attenuates the oscillation between source graphs (Proposition~\ref{flo:prop:slow}).

\section{Experiments}\label{sec:experiments}
In this section, we evaluate FLO on six node classification, three link prediction, and three graph classification datasets. We aim to answer the following four research questions. \textbf{Q1}: Does FLO transfer effectively to unseen graph domains? \textbf{Q2}: Does FLO learn generalizable knowledge that transfers to unseen graph tasks? \textbf{Q3}: Does FLO benefit from increasing domain diversity while keeping pre-training stable? \textbf{Q4}: How do reciprocal relation learning and reconstruction contribute to the performance of FLO? Further results are given in Appendix~\ref{app:experiments}.

\subsection{Experimental Setup}\label{sec:setup}
\textbf{Datasets.} For node classification, we use six graphs: the citation networks Cora, CiteSeer, and PubMed \citep{sen2008collective,yang2016revisiting}, the co-authorship network CS, and the co-purchase networks Photo and Computers \citep{shchur2018pitfalls}. For cross-task transfer, we use three link prediction datasets, Physics \citep{shchur2018pitfalls}, WikiCS \citep{mernyei2020wiki}, and DBLP \citep{bojchevski2018deep}, and three graph classification datasets, COIL-RAG, ENZYMES, and PROTEINS \citep{morris2020tudataset}. Dataset statistics are given in Appendix~\ref{app:datasets}.

\textbf{Baselines.} We compare FLO with ten methods in three categories: (1) graph self-supervised learning methods: DGI \citep{veličković2018deep}, GRACE \citep{zhu2020deep}, and BGRL \citep{thakoor2022largescale}; (2) feature unification methods: FUG \citep{zhao2024fug} and TIG \citep{zhao2026towards}; and (3) multi-domain graph pre-training methods: GFT \citep{wang2024gft}, MDGPT \citep{yu2024text}, SAMGPT \citep{yu2025samgpt}, MDGFM \citep{wang2025multidomain}, and BRIDGE \citep{yuan2025much}. All baselines are run with their official code and the settings recommended in their papers, as detailed in Appendix~\ref{app:baselines}, including the input preprocessing each method requires.

\textbf{Evaluation protocol.} All encoders are frozen during downstream evaluation, and for each task all methods use the same downstream classifier and the same data splits. For node classification, following FUG \citep{zhao2024fug}, we train an $\ell_2$-regularized logistic regression classifier with a 10\%/10\%/80\% training/validation/test split of the target nodes, and report the mean and standard deviation over three random splits. For few-shot node classification, following TIG \citep{zhao2026towards}, we randomly label $k\in\{1,3,5\}$ nodes per class, train the same classifier, and report the mean and standard deviation over 500 sampled tasks shared by all methods. FLO shares one set of pre-training hyperparameters across all target graphs. We evaluate the final encoder without selecting checkpoints on the target graph. The protocol of each task and the hyperparameters of FLO are given in Appendices~\ref{app:probe} and~\ref{app:implementation}. The random seed is fixed to 0.

\begin{table}[t]
\centering
\footnotesize
\setlength{\tabcolsep}{1.5pt}
\caption{Accuracy (\%) of cross-domain node classification with standard deviations. Each column is an unseen target graph. The best and second-best results are in bold and underlined.}\label{tab:nc}
\begin{tabular*}{\linewidth}{@{\extracolsep{\fill}}lccccccc}
\toprule
Method & Cora & CiteSeer & PubMed & CS & Photo & Computers & Avg. \\
\midrule
DGI & 75.47$_{\pm0.78}$ & 65.76$_{\pm1.19}$ & 82.83$_{\pm0.24}$ & 92.82$_{\pm0.18}$ & 91.48$_{\pm0.26}$ & 87.24$_{\pm0.21}$ & 82.60 \\
GRACE & 72.00$_{\pm0.36}$ & 56.29$_{\pm0.71}$ & 76.80$_{\pm0.37}$ & 90.29$_{\pm0.09}$ & 90.76$_{\pm0.15}$ & 86.82$_{\pm0.15}$ & 78.83 \\
BGRL & 79.43$_{\pm1.24}$ & 64.46$_{\pm1.14}$ & 83.91$_{\pm0.26}$ & 92.63$_{\pm0.21}$ & 92.76$_{\pm0.21}$ & 88.98$_{\pm0.16}$ & 83.69 \\
\midrule
FUG & 82.79$_{\pm0.72}$ & 70.52$_{\pm0.15}$ & 84.42$_{\pm0.23}$ & 92.41$_{\pm0.02}$ & 92.31$_{\pm0.16}$ & 87.30$_{\pm0.16}$ & 84.96 \\
TIG & \underline{84.39}$_{\pm1.08}$ & \textbf{73.08}$_{\pm0.52}$ & \underline{86.45}$_{\pm0.13}$ & \underline{93.42}$_{\pm0.05}$ & 93.27$_{\pm0.22}$ & 89.51$_{\pm0.07}$ & \underline{86.68} \\
GFT & 72.99$_{\pm0.94}$ & 64.79$_{\pm0.99}$ & 83.59$_{\pm0.12}$ & 93.06$_{\pm0.27}$ & 90.25$_{\pm0.18}$ & 84.18$_{\pm0.19}$ & 81.48 \\
MDGPT & 81.78$_{\pm0.47}$ & 72.25$_{\pm0.49}$ & 85.76$_{\pm0.24}$ & 93.02$_{\pm0.08}$ & 93.27$_{\pm0.07}$ & 90.01$_{\pm0.09}$ & 86.02 \\
SAMGPT & 81.64$_{\pm0.56}$ & 71.59$_{\pm0.34}$ & 85.44$_{\pm0.23}$ & 93.04$_{\pm0.09}$ & 92.98$_{\pm0.20}$ & 89.91$_{\pm0.16}$ & 85.77 \\
MDGFM & 81.60$_{\pm0.70}$ & 71.66$_{\pm0.61}$ & 85.63$_{\pm0.25}$ & 92.84$_{\pm0.09}$ & \underline{93.33}$_{\pm0.15}$ & \textbf{90.19}$_{\pm0.13}$ & 85.87 \\
BRIDGE & 81.59$_{\pm0.65}$ & 71.36$_{\pm0.51}$ & 85.47$_{\pm0.23}$ & 93.11$_{\pm0.05}$ & 93.20$_{\pm0.13}$ & 90.07$_{\pm0.20}$ & 85.80 \\
\midrule
\textbf{FLO} & \textbf{85.83}$_{\pm0.65}$ & \underline{72.50}$_{\pm0.91}$ & \textbf{86.74}$_{\pm0.35}$ & \textbf{93.62}$_{\pm0.05}$ & \textbf{93.62}$_{\pm0.23}$ & \underline{90.13}$_{\pm0.07}$ & \textbf{87.07} \\
\bottomrule
\end{tabular*}
\end{table}

\subsection{Cross-Domain Transfer (Q1)}\label{sec:q1}
To answer Q1, we adopt the leave-one-domain-out setting: each of the six graphs is held out in turn as the target graph, and the model is pre-trained without labels on the other five. \textbf{FLO gives the best average accuracy on unseen graph domains in both standard and few-shot node classification.}

\textbf{Node classification.} As shown in Table~\ref{tab:nc}, FLO achieves the best average accuracy, 87.07\%, and ranks first on four of the six target graphs. Feature unification methods align the input features of different graphs, and their encoders learn node relations only through message passing. FLO further refines these relations with those predicted from node attributes alone, so its encoder also acquires the knowledge beyond propagation (Theorem~\ref{flo:thm:recip}). Without this reciprocal relation learning, FLO falls below the best feature unification method on Cora and CiteSeer (Table~\ref{tab:ablation}). Multi-domain graph pre-training methods such as MDGPT and SAMGPT align domains with randomly initialized parameters updated by immediate loss feedback. FLO introduces no such parameters and takes its relation targets from slow copies, which filter cross-domain interference (Proposition~\ref{flo:prop:slow}).

\textbf{Few-shot node classification.} FLO achieves the best average accuracy with one, three, and five labeled nodes per class (Figure~\ref{fig:fewshot}), exceeding the strongest baseline by 3.71, 4.63, and 2.92 points, respectively. These margins are larger than in node classification, since with few labels the accuracy depends more on the relations among node representations, which FLO refines with the knowledge beyond propagation. The results of all methods on each target graph are given in Appendix~\ref{app:fewshot}.

\begin{figure}[t]
\centering
\includegraphics[width=\linewidth]{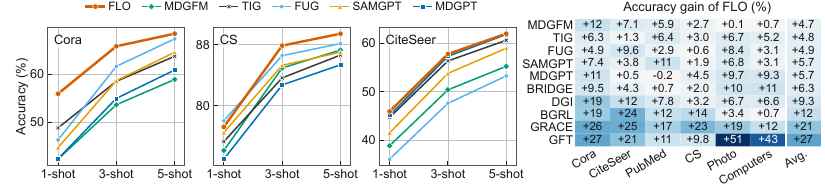}
\caption{Few-shot node classification. Left: accuracy (\%) on Cora, CS, and CiteSeer for FLO and the five baselines with the highest average accuracy. Right: accuracy gain (\%) of FLO over each baseline, averaged over the 1-, 3-, and 5-shot settings; Avg. is the mean over the six target graphs.}
\label{fig:fewshot}
\end{figure}

\subsection{Cross-Task Transfer (Q2)}\label{sec:q2}
To answer Q2, we pre-train the encoder only on the six node classification datasets and apply it, frozen, to three link prediction and three graph classification datasets. We compare FLO with the feature unification and multi-domain graph pre-training methods, which are designed for transfer across graphs. \textbf{FLO learns generalizable knowledge that transfers to unseen tasks on unseen graphs.} As shown in Table~\ref{tab:cross-task}, FLO achieves the best average on both tasks, and its margin is largest on link prediction, 9.00 points in Hits@50 over the runner-up. Link prediction scores pairs of nodes, and FLO pre-trains the encoder to predict the relations between an anchor and its candidates. On ENZYMES and PROTEINS, whose nodes have only 21 and 32 features, FLO trails the best baseline by 0.40 and 0.58 points, within the standard deviations. One possible reason is that FLO learns the knowledge beyond propagation from node features, which are scarce on these graphs.

\begin{table}[H]
\centering
\footnotesize
\setlength{\tabcolsep}{1.3pt}
\caption{Hits@50 (\%) of link prediction and accuracy (\%) of graph classification with standard deviations. The best and second-best results are in bold and underlined.}
\label{tab:cross-task}
\begin{tabular*}{\linewidth}{@{\extracolsep{\fill}}lcccc@{\hspace{12pt}}cccc}
\toprule
& \multicolumn{4}{c}{Link prediction: Hits@50 (\%)} & \multicolumn{4}{c}{Graph classification: accuracy (\%)} \\
\cmidrule(lr){2-5}\cmidrule(lr){6-9}
Method & Physics & WikiCS & DBLP & Avg. & COIL-RAG & ENZYMES & PROTEINS & Avg. \\
\midrule
FUG & \underline{55.99}$_{\pm3.38}$ & 66.23$_{\pm1.22}$ & \underline{64.18}$_{\pm5.29}$ & \underline{62.13} & \underline{84.97}$_{\pm1.50}$ & 33.41$_{\pm4.31}$ & 71.64$_{\pm1.66}$ & 63.34 \\
TIG & 44.46$_{\pm2.88}$ & 69.93$_{\pm0.59}$ & 61.85$_{\pm3.43}$ & 58.75 & 77.16$_{\pm1.53}$ & 33.33$_{\pm3.36}$ & 71.56$_{\pm2.36}$ & 60.68 \\
GFT & 42.93$_{\pm2.90}$ & 58.78$_{\pm2.13}$ & 43.41$_{\pm5.49}$ & 48.37 & 83.42$_{\pm1.62}$ & \textbf{35.88}$_{\pm5.91}$ & \underline{72.37}$_{\pm2.92}$ & \underline{63.89} \\
MDGPT & 47.79$_{\pm1.52}$ & 68.03$_{\pm1.26}$ & 57.25$_{\pm1.56}$ & 57.69 & 80.14$_{\pm2.19}$ & 31.24$_{\pm4.22}$ & 71.59$_{\pm3.22}$ & 60.99 \\
SAMGPT & 49.82$_{\pm1.86}$ & 70.12$_{\pm0.73}$ & 57.60$_{\pm1.75}$ & 59.18 & 81.91$_{\pm1.99}$ & 29.57$_{\pm4.13}$ & 70.65$_{\pm2.71}$ & 60.71 \\
MDGFM & 51.22$_{\pm2.39}$ & \underline{70.92}$_{\pm1.23}$ & 55.09$_{\pm6.77}$ & 59.08 & 81.86$_{\pm1.91}$ & 33.73$_{\pm4.32}$ & \textbf{72.54}$_{\pm2.07}$ & 62.71 \\
BRIDGE & 49.97$_{\pm3.30}$ & 69.10$_{\pm0.80}$ & 56.94$_{\pm1.58}$ & 58.67 & 82.26$_{\pm1.80}$ & 31.24$_{\pm3.59}$ & 71.14$_{\pm3.09}$ & 61.55 \\
\midrule
\textbf{FLO} & \textbf{67.37}$_{\pm9.49}$ & \textbf{72.75}$_{\pm0.98}$ & \textbf{73.28}$_{\pm0.78}$ & \textbf{71.13} & \textbf{86.59}$_{\pm1.35}$ & \underline{35.48}$_{\pm4.65}$ & 71.96$_{\pm2.12}$ & \textbf{64.68} \\
\bottomrule
\end{tabular*}
\end{table}

\subsection{Domain Diversity and Training Stability (Q3)}\label{sec:q3}

To answer Q3, we vary the number of source domains and track the pre-training loss.

\noindent\textbf{Domain diversity.} We add source domains in a fixed order, skipping the target graph, and keep the pre-training budget fixed. As shown in Figure~\ref{fig:source-count-six}(a), from one to five source domains, FLO improves by 1.37--6.31 points on all six target graphs, whereas the four multi-domain baselines change by at most 0.63 points. \textbf{FLO benefits from increasing domain diversity.}

\begin{figure}[t]
\centering
\includegraphics[width=\linewidth]{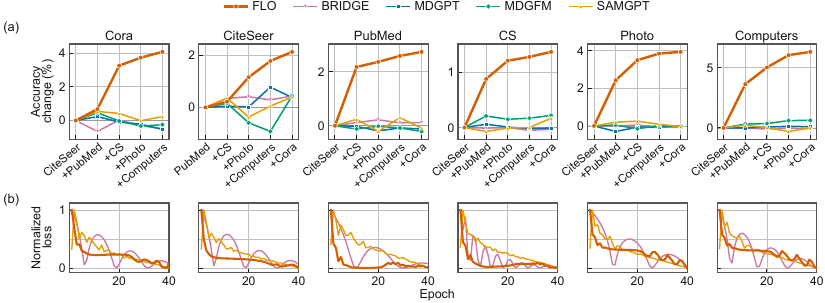}
\caption{Domain diversity and training stability. (a) Accuracy change (\%) as source domains are added in the order on the horizontal axis. (b) Min--max normalized log loss in the first 40 epochs.}
\label{fig:source-count-six}
\end{figure}

\noindent\textbf{Training stability.} We record the pre-training loss in the first 40 epochs, taking the loss of the graph encoder for FLO. As shown in Figure~\ref{fig:source-count-six}(b), the loss of FLO rises by at most 6.6\% between consecutive epochs, whereas those of BRIDGE and SAMGPT rise by up to 249\% and 371\%. \textbf{FLO keeps pre-training stable as the source domains alternate.} These observations are complementary: reduced loss oscillation alone would not establish effective knowledge acquisition, while the consistent gains from additional source domains show that stability is accompanied by improved transfer. This pattern matches the intended roles of PARR and FSKA: PARR exposes complementary relational information, while FSKA consolidates it across domains.

\subsection{Ablation Study (Q4)}\label{sec:q4}
To answer Q4, we remove or replace parts of FLO (Table~\ref{tab:ablation}). \textbf{Reciprocal relation learning improves the graph encoder through the returned relation targets, and reconstruction rules out the trivial solution of relation learning.} Removing both relation learning objectives lowers accuracy most on CiteSeer and Cora, the two graphs on which propagation changes the relations most (Section~\ref{flo:sec:theory}). Removing the slow attribute branch gives accuracy similar to removing both objectives, since the graph encoder then receives no returned targets (Theorem~\ref{flo:thm:recip}(ii)). Removing the slow graph branch lowers accuracy further, since the attribute branch is then not driven toward $\overline P$ (Theorem~\ref{flo:thm:recip}(i)) and its returned targets no longer correct the graph encoder. Removing reconstruction lowers accuracy to 30.33\% on Cora and 20.66\% on CiteSeer, close to the share of the largest class (Proposition~\ref{flo:prop:recon}). The two directions thus serve distinct roles: graph-to-attribute transfer identifies the propagation-free conditional target, whereas attribute-to-graph transfer returns this target to the transferable graph encoder. A GCN attribute branch gives no gain over removing both relation learning objectives on any target graph, so the gain depends on a propagation-free attribute branch.

\begin{table}[H]
\centering
\footnotesize
\setlength{\tabcolsep}{1pt}
\caption{Ablation study on node classification: accuracy (\%) with standard deviations. w/o slow attribute branch removes $\mathcal{L}^{G}_{\rm kd}$, the relation targets returned to the graph branch; w/o slow graph branch removes $\mathcal{L}^{M}_{\rm kd}$, the relation targets that train the attribute branch. Removing one doubles the weight of the other, so the total relation weight is unchanged. The best results are in bold.}\label{tab:ablation}
\begin{tabular*}{\linewidth}{@{\extracolsep{\fill}}lcccccc}
\toprule
Variant & Cora & CiteSeer & PubMed & CS & Photo & Computers \\
\midrule
w/o $\mathcal{L}_{\rm kd}$ & 83.26$_{\pm0.84}$ & 66.52$_{\pm1.31}$ & 86.63$_{\pm0.23}$ & 93.13$_{\pm0.09}$ & 93.34$_{\pm0.07}$ & 90.04$_{\pm0.12}$ \\
w/o slow attribute branch & 83.37$_{\pm0.89}$ & 66.98$_{\pm1.11}$ & 86.67$_{\pm0.26}$ & 93.03$_{\pm0.14}$ & 93.47$_{\pm0.15}$ & 90.04$_{\pm0.11}$ \\
w/o slow graph branch & 64.65$_{\pm0.95}$ & 42.01$_{\pm1.61}$ & 80.57$_{\pm0.32}$ & 82.58$_{\pm0.23}$ & 89.62$_{\pm0.36}$ & 84.83$_{\pm0.20}$ \\
w/o $\mathcal{L}_{\rm rec}$ & 30.33$_{\pm0.19}$ & 20.66$_{\pm0.63}$ & 82.48$_{\pm0.31}$ & 76.14$_{\pm0.92}$ & 59.72$_{\pm1.94}$ & 52.40$_{\pm1.16}$ \\
w/ GCN attribute branch & 79.74$_{\pm0.10}$ & 55.56$_{\pm0.22}$ & 86.60$_{\pm0.26}$ & 91.52$_{\pm0.15}$ & 92.02$_{\pm0.55}$ & 89.48$_{\pm0.17}$ \\
\midrule
\textbf{FLO} & \textbf{85.83}$_{\pm0.65}$ & \textbf{72.50}$_{\pm0.91}$ & \textbf{86.74}$_{\pm0.35}$ & \textbf{93.62}$_{\pm0.05}$ & \textbf{93.62}$_{\pm0.23}$ & \textbf{90.13}$_{\pm0.07}$ \\
\bottomrule
\end{tabular*}
\end{table}

\FloatBarrier
\section{Conclusion}
In this paper, we proposed FLO for multi-domain graph pre-training, where existing methods struggle to benefit from increasing domain diversity. PARR refines the node relations induced by message passing with feedback grounded in node attributes, and FSKA takes this feedback from a slow process that integrates knowledge across domains. We proved that PARR leads the graph encoder to the knowledge beyond propagation and FSKA filters cross-domain interference. On 12 datasets, FLO achieves the best average performance in cross-domain and cross-task transfer and benefits from increasing domain diversity, suggesting that domain alignment is a prerequisite rather than a guarantee of knowledge acquisition. Future work will extend FLO to graphs with scarce node features.

\bibliography{flo_references}
\bibliographystyle{flo_preprint}
\clearpage

\appendix
\section{Proofs}
\label{flo:app:theory}
\label{app:theory}
We prove the results of Section~\ref{flo:sec:theory} and one result on reconstruction. Table~\ref{tab:proof-map} lists what each result shows and the main tools of its proof. Appendix~\ref{flo:app:setting} fixes the setting and two facts used throughout. Appendix~\ref{app:relation-change} relates the relation change measured on our graphs to the gain of PARR, and Appendix~\ref{flo:app:invariance} states the invariance of relation distributions.

\begin{table}[H]
\centering\footnotesize
\caption{Results proved in this appendix.}
\label{tab:proof-map}
\begin{tabularx}{\linewidth}{@{}lXXl@{}}
\toprule
Result & What it shows & Main tools & Proof\\
\midrule
Eq.~\eqref{flo:eq:decomp} & The relational information of the graph branch splits into an attribute-predicted part and the propagation-induced part $R$ & Chain rule, tower property & \ref{flo:app:decomp}\\
Lemma~\ref{flo:lem:stable} & Close embeddings give close relation distributions & Lipschitz bounds of cosine and softmax & \ref{flo:app:stable}\\
Theorem~\ref{flo:thm:prop} & Propagation drives the relations among nearby nodes toward uniform & Edge form of the Dirichlet energy, spectral decomposition & \ref{flo:app:prop}\\
Theorem~\ref{flo:thm:recip} & Targets returned by the attribute branch remove $R$ & Conditional mean, convexity of $\operatorname{KL}$, Pinsker's inequality & \ref{flo:app:recip}\\
Proposition~\ref{flo:prop:recon} & Reconstruction rules out uniform relations & Bounds of the decoder on directions & \ref{flo:app:recon}\\
Proposition~\ref{flo:prop:slow} & Slow copies damp period-$S$ oscillations by $\rho$ & Frequency response, Parseval's identity & \ref{flo:app:slow}\\
\bottomrule
\end{tabularx}
\end{table}

\subsection{Setting and Notation}
\label{flo:app:setting}
Expectations are taken over the sampled source graph $s$, anchor, candidate slots, and masks, with the slow copies held fixed in Theorems~\ref{flo:thm:prop} and~\ref{flo:thm:recip}. The observation $\xi$ includes the masked input of the attribute branch, so the targets of the attribute branch are functions of $\xi$. Relation distributions compared in a $\operatorname{KL}$ divergence share their support and are positive on the valid slots.

For graph $s$ we write $\widetilde{\mathbf A}=\mathbf A_s+\mathbf I$, $\widetilde d_k=\sum_l\widetilde{\mathbf A}_{kl}$, $\widehat{\mathbf A}_s=\sum_r\omega_r\mathbf u_r\mathbf u_r^\top$, and $\mathbf L=\mathbf I-\widehat{\mathbf A}_s$, whose eigenvalues $1-\omega_r$ lie in $[0,2)$. For a nonzero vector $\mathbf z$, $\hat{\mathbf z}=\mathbf z/\|\mathbf z\|_2$. Theorem~\ref{flo:thm:prop} treats linear propagation. ReLU does not increase the Dirichlet energy and a weight matrix scales it by at most $\|\mathbf W\|_{\mathrm{op}}^2$, so the contraction extends to bias-free ReLU layers \citep{cai2020note}.

\textbf{Two facts.} For nonzero $\mathbf u,\mathbf b$:
\begin{equation}
\|\hat{\mathbf u}-\hat{\mathbf b}\|_2\le\frac{2\|\mathbf u-\mathbf b\|_2}{\|\mathbf b\|_2},
\label{flo:eq:normfact}
\end{equation}
which follows from $\|\hat{\mathbf u}-\hat{\mathbf b}\|\le\bigl|\|\mathbf b\|-\|\mathbf u\|\bigr|/\|\mathbf b\|+\|\mathbf u-\mathbf b\|/\|\mathbf b\|$. For nonzero vectors, $\cos(\mathbf u,\mathbf b)=1-\tfrac12\|\hat{\mathbf u}-\hat{\mathbf b}\|_2^2$.

\subsection{Information Decomposition}
\label{flo:app:decomp}
Eq.~\eqref{flo:eq:decomp} separates what node attributes predict from what propagation adds. We split the information by the chain rule and evaluate the conditional part with the tower property.

\begin{restate}{Eq.~\eqref{flo:eq:decomp}}
$I(C;P)=I(C;\xi)+R$ with $R=I(C;P\mid\xi)=\mathbb E\,\operatorname{KL}(P\|\overline P)$.
\end{restate}

\begin{proof}\textbf{Step 1 (chain rule).} Given $P$, the slot $C$ is drawn from $P$ and does not depend on $\xi$, so $I(C;\xi\mid P)=0$ and $I(C;P)=I(C;P,\xi)=I(C;\xi)+I(C;P\mid\xi)$.

\textbf{Step 2 (conditional entropy).} Given $\xi$ alone, $C$ follows $\mathbb E[P\mid\xi]=\overline P$. Hence $I(C;P\mid\xi)=H(C\mid\xi)-H(C\mid P,\xi)=\mathbb E H(\overline P)-\mathbb E H(P)$.

\textbf{Step 3 (tower property).} Since $\overline P$ is a function of $\xi$, $\mathbb E\sum_jP_j\log\overline P_j=\mathbb E\sum_j\overline P_j\log\overline P_j=-\mathbb E H(\overline P)$. Therefore $\mathbb E\,\operatorname{KL}(P\|\overline P)=-\mathbb E H(P)+\mathbb E H(\overline P)=I(C;P\mid\xi)$.
\end{proof}

\subsection{A Lemma on Stable Relations}
\label{flo:app:stable}
Theorem~\ref{flo:thm:recip}(iii) and Proposition~\ref{flo:prop:slow} compare relation distributions computed from two nearby sets of embeddings. Small changes of the embeddings give small changes of every cosine, and the softmax turns small changes of the logits into small changes of the log-probabilities.

\begin{lemma}[Stable relations]
\label{flo:lem:stable}
Let $\mathbf u_k,\mathbf b_k$ be two embeddings of the anchor ($k=0$) and its candidates, with $\|\mathbf u_k-\mathbf b_k\|_2\le\epsilon$, $\|\mathbf b_k\|_2\ge c>0$, and $\mathbf u_k\neq\mathbf 0$ for all $k$. Then their temperature-$\tau$ relation distributions satisfy:
\[
\max\bigl\{\operatorname{KL}(p_{\mathbf U}\|p_{\mathbf B}),\,\operatorname{KL}(p_{\mathbf B}\|p_{\mathbf U})\bigr\}\le\frac{8\epsilon}{\tau c}.
\]
\end{lemma}

\begin{proof}\textbf{Step 1 (directions).} By Eq.~\eqref{flo:eq:normfact}, $\|\hat{\mathbf u}_k-\hat{\mathbf b}_k\|\le2\epsilon/c$ for every $k$.

\textbf{Step 2 (cosines and logits).} For unit vectors, $|\langle\mathbf a,\mathbf b\rangle-\langle\mathbf a',\mathbf b'\rangle|\le\|\mathbf a-\mathbf a'\|+\|\mathbf b-\mathbf b'\|$. Hence every cosine between the anchor and a candidate changes by at most $4\epsilon/c$, and every logit by at most $\Delta=4\epsilon/(\tau c)$.

\textbf{Step 3 (log-probabilities).} A log-probability is a logit minus the log-sum-exp of all logits, and the log-sum-exp changes by at most the largest change of a logit. Hence every log-probability changes by at most $2\Delta$, and so does either $\operatorname{KL}$ divergence, which is an average of these differences.
\end{proof}

\subsection{Proof of Theorem~\ref{flo:thm:prop}}
\label{flo:app:prop}
Theorem~\ref{flo:thm:prop} formalizes how propagation suppresses relational variation among nearby candidates. Small energy means that nearby nodes look alike. The anchor then cannot prefer any nearby candidate over itself, so its relations are close to uniform. Steps 1--4 bound the relations by the energy, Step 5 bounds the energy by the spectrum, and Step 6 turns the bound into a bound on information.

\begin{restate}{Theorem~\ref{flo:thm:prop}}
Let $\mathbf Z=\widehat{\mathbf A}_s^{K}\mathbf Y$, and let $P_i^h$ be the relations of anchor $i$ over its candidates within $h$ hops. Then the following holds:
\[
\operatorname{KL}(P_i^h\|U)\le \beta_i\,\mathcal E(\mathbf Z)\le \beta_i\,\zeta_s^{2K}\,\mathcal E(\mathbf Y),
\]
where $U$ is uniform and $\beta_i=2h\widetilde d_i/(\tau\|\mathbf z_i\|_2^2)$. Moreover, $I(C;P_i^h)\le\mathbb E\,\operatorname{KL}(P_i^h\|U)$.
\end{restate}

\begin{proof}\textbf{Step 1 (energy as edge differences).} With $\mathbf y_k=\mathbf z_k/\sqrt{\widetilde d_k}$:
\[
\mathcal E(\mathbf Z)=\operatorname{tr}(\mathbf Z^\top\mathbf L\mathbf Z)=\tfrac12\sum_{k,l}\widetilde{\mathbf A}_{kl}\|\mathbf y_k-\mathbf y_l\|_2^2=\sum_{\{k,l\}\in E_s}\|\mathbf y_k-\mathbf y_l\|_2^2 .
\]

\textbf{Step 2 (nearby nodes).} Let $v$ be a candidate within $h$ hops and take a shortest path from $i$ to $v$. Its edges are distinct, so by the Cauchy--Schwarz inequality $\|\mathbf y_i-\mathbf y_v\|_2\le\sqrt{h\,\mathcal E(\mathbf Z)}$.

\textbf{Step 3 (cosines).} Cosines ignore positive scaling, so $\hat{\mathbf z}_k=\hat{\mathbf y}_k$. By Eq.~\eqref{flo:eq:normfact} and $\|\mathbf y_i\|=\|\mathbf z_i\|/\sqrt{\widetilde d_i}$:
\[
\|\hat{\mathbf z}_i-\hat{\mathbf z}_v\|_2\le\frac{2\sqrt{\widetilde d_i\,h\,\mathcal E(\mathbf Z)}}{\|\mathbf z_i\|_2},
\qquad
\cos(\mathbf z_i,\mathbf z_v)\ge1-\frac{2h\widetilde d_i\,\mathcal E(\mathbf Z)}{\|\mathbf z_i\|_2^2}.
\]

\textbf{Step 4 (near-uniform relations).} The anchor slot has cosine one, so all $m$ logits lie in $[1/\tau-w,\,1/\tau]$ with $w=\beta_i\mathcal E(\mathbf Z)$. Every probability is then at most $e^{w}/m$, and $\operatorname{KL}(P_i^h\|U)=\sum_jp_j\log(mp_j)\le w$.

\textbf{Step 5 (energy contraction).} $\widehat{\mathbf A}_s$ and $\mathbf L$ share eigenvectors, so:
\[
\mathcal E(\widehat{\mathbf A}_s^K\mathbf Y)=\sum_r(1-\omega_r)\,\omega_r^{2K}\|\mathbf u_r^\top\mathbf Y\|_2^2\le\zeta_s^{2K}\sum_r(1-\omega_r)\|\mathbf u_r^\top\mathbf Y\|_2^2=\zeta_s^{2K}\mathcal E(\mathbf Y),
\]
because the terms with $\omega_r=1$ vanish and all other eigenvalues satisfy $|\omega_r|\le\zeta_s$.

\textbf{Step 6 (information).} For a random distribution $P$ and $C$ drawn from it, $\mathbb E\,\operatorname{KL}(P\|U)=\log m-\mathbb E H(P)=\operatorname{KL}(\mathbb E P\|U)+I(C;P)\ge I(C;P)$.
\end{proof}

\subsection{Proof of Theorem~\ref{flo:thm:recip}}
\label{flo:app:recip}
Theorem~\ref{flo:thm:recip} is the basis of PARR. The attribute branch sees only $\xi$, so the best it can do is the conditional mean $\overline P$. Mixing the graph relations with any function of $\xi$ shrinks their distance to that mean. Part (i) identifies the best attribute target, Part (ii) shows what a returned target removes, and Part (iii) relates $R$ to the Dirichlet energy.

\begin{restate}{Theorem~\ref{flo:thm:recip}}
(i) For every attribute branch $Q$, $\;\mathbb E\,\operatorname{KL}(P\|Q)=R+\mathbb E\,\operatorname{KL}(\overline P\|Q)\ge R$, with equality if and only if $Q=\overline P$.
(ii) If the relations of the graph branch move toward a returned target $Q$, that is, $P'=(1-\eta)P+\eta Q$ with $\eta\in[0,1]$, then $R'\le(1-\eta)R$. If $Q=\overline P$, then in addition $I(C';\xi)=I(C;\xi)$.
(iii) $\mathbb E\|P-\overline P\|_1^2\le2R$, and under the conditions of Step 2 of Part (iii) below, $R\le\kappa\sum_s\pi_s\sqrt{\mathcal E(\mathbf X_s)}$.
\end{restate}

\begin{proof}\textbf{Part (i).} $Q$ and $\overline P$ are functions of $\xi$. Replacing $P$ by $\mathbb E[P\mid\xi]=\overline P$ inside the expectation gives:
\[
\mathbb E\,\operatorname{KL}(P\|Q)-\mathbb E\,\operatorname{KL}(P\|\overline P)=\mathbb E\sum_jP_j\log\frac{\overline P_j}{Q_j}=\mathbb E\sum_j\overline P_j\log\frac{\overline P_j}{Q_j}=\mathbb E\,\operatorname{KL}(\overline P\|Q)\ge0,
\]
with equality if and only if $Q=\overline P$ almost surely.

\textbf{Part (ii).} \emph{Step 1 (new attribute-predicted relations).} Since $Q$ is a function of $\xi$, the new attribute-predicted relations are $\overline P'=\mathbb E[P'\mid\xi]=(1-\eta)\overline P+\eta Q$.

\emph{Step 2 (convexity).} The $\operatorname{KL}$ divergence is jointly convex, so for every realization:
\[
\operatorname{KL}\bigl((1-\eta)P+\eta Q\,\big\|\,(1-\eta)\overline P+\eta Q\bigr)\le(1-\eta)\operatorname{KL}(P\|\overline P)+\eta\,\operatorname{KL}(Q\|Q)=(1-\eta)\operatorname{KL}(P\|\overline P).
\]
Taking expectations gives $R'\le(1-\eta)R$.

\emph{Step 3 (the attribute-predicted part is kept).} If $Q=\overline P$, then $\overline P'=\overline P$. The attribute-predicted information $I(C;\xi)=\mathbb E\,\operatorname{KL}(\overline P\|\mathbb E\overline P)$ depends only on $\overline P$ and is therefore unchanged.

\textbf{Part (iii).} \emph{Step 1 (Pinsker's inequality).} Pinsker's inequality gives $\|P-\overline P\|_1^2\le2\operatorname{KL}(P\|\overline P)$; taking expectations gives $\mathbb E\|P-\overline P\|_1^2\le2R$.

\emph{Step 2 (a target computed without propagation).} Assume that the graph branch computes relations from $\mathbf U=\widehat{\mathbf A}_s\mathbf X_s\mathbf W$, that $\xi$ determines the rows of $\mathbf B=\mathbf X_s\mathbf W$ for the anchor and its candidates, and that these rows have norm at least $c$. The relations $Q_B$ computed from $\mathbf B$ are a function of $\xi$, so Part (i) gives $R\le\mathbb E\,\operatorname{KL}(P\|Q_B)$.

\emph{Step 3 (row changes).} Each row changes by $\|(\mathbf U-\mathbf B)_k\|=\|(\mathbf L\mathbf X_s\mathbf W)_k\|\le\|\mathbf W\|_{\mathrm{op}}\|\mathbf L\mathbf X_s\|_F$, and $\|\mathbf L\mathbf X_s\|_F^2=\operatorname{tr}(\mathbf X_s^\top\mathbf L^2\mathbf X_s)\le2\,\mathcal E(\mathbf X_s)$ because the eigenvalues of $\mathbf L$ lie in $[0,2)$.

\emph{Step 4 (averaging over graphs).} Lemma~\ref{flo:lem:stable} with $\epsilon=\|\mathbf W\|_{\mathrm{op}}\sqrt{2\mathcal E(\mathbf X_s)}$ bounds $\operatorname{KL}(P\|Q_B)$ on graph $s$. Averaging over $s$ gives $R\le\kappa\sum_s\pi_s\sqrt{\mathcal E(\mathbf X_s)}$ with $\kappa=8\sqrt2\,\|\mathbf W\|_{\mathrm{op}}/(\tau c)$.
\end{proof}

\textbf{Remark (a non-ideal target).} For logits $\mathbf a_G$ with $p_G=\operatorname{softmax}(\mathbf a_G)$, a returned target $Q$ induces the gradient $p_G-Q$ of $\operatorname{KL}(Q\|p_G)$. It differs from the gradient $p_G-\overline P$ of the ideal target by $\overline P-Q$, and Pinsker's inequality with Part (i) gives $\mathbb E\|\overline P-Q\|_2^2\le2\,\mathbb E\,\operatorname{KL}(\overline P\|Q)=2\bigl[\mathbb E\,\operatorname{KL}(P\|Q)-R\bigr]$. The error of the returned target is therefore set by how well the attribute branch fits the graph branch.

\subsection{Reconstruction Rules Out Uniform Relations}
\label{flo:app:recon}
Relation learning alone admits the trivial solution of uniform relations (Section~\ref{flo:sec:parr}). The proposition below shows that reconstruction rules it out. If two embeddings pointed in the same direction, they would decode to the same feature.

\begin{proposition}
\label{flo:prop:recon}
Let a branch decode the unit feature $\mathbf x_k$ as $\operatorname{RowNorm}(D\mathbf z_k)$ with a linear decoder $D$ of smallest singular value $s_{\min}>0$, and let $\epsilon_k=\|\operatorname{RowNorm}(D\mathbf z_k)-\mathbf x_k\|_2$. Then the following holds:
\[
\|\hat{\mathbf z}_i-\hat{\mathbf z}_j\|_2\ge\frac{s_{\min}}{2\|D\|_{\mathrm{op}}}\bigl(\|\mathbf x_i-\mathbf x_j\|_2-\epsilon_i-\epsilon_j\bigr).
\]
\end{proposition}

\begin{proof}\textbf{Step 1 (the decoder on directions).} Let $g(\mathbf z)=D\mathbf z/\|D\mathbf z\|_2$. It ignores positive scaling, so $g(\mathbf z)=g(\hat{\mathbf z})$. By Eq.~\eqref{flo:eq:normfact} and $\|D\hat{\mathbf z}_j\|\ge s_{\min}$, $\|g(\hat{\mathbf z}_i)-g(\hat{\mathbf z}_j)\|\le2\|D\|_{\mathrm{op}}\|\hat{\mathbf z}_i-\hat{\mathbf z}_j\|/s_{\min}$.

\textbf{Step 2 (triangle inequality).} The triangle inequality gives $\|\mathbf x_i-\mathbf x_j\|\le\epsilon_i+\epsilon_j+\|g(\hat{\mathbf z}_i)-g(\hat{\mathbf z}_j)\|$. Combining with Step 1 proves the claim.
\end{proof}

Uniform relations would give every candidate the same cosine as the anchor slot, namely one, and hence $\hat{\mathbf z}_j=\hat{\mathbf z}_i$. The proposition rules this out for every candidate whose features differ from those of the anchor by more than $\epsilon_i+\epsilon_j$.

\subsection{Proof of Proposition~\ref{flo:prop:slow}}
\label{flo:app:slow}
Proposition~\ref{flo:prop:slow} is the basis of FSKA. The slow update is a linear filter. It passes slow trends and damps oscillations whose period is short compared with $1/(1-\mu)$ steps. Steps 1--2 separate the shared trend from the oscillation, Steps 3--5 treat each part, and Step 6 carries the bound to the relation targets.

\begin{restate}{Proposition~\ref{flo:prop:slow}}
Let the source graphs be visited in turn, $\delta_t=\mathbf g+\mathbf e_{s(t)}$ with $\sum_s\mathbf e_s=\mathbf 0$, and $\overline\theta_t=\mu\overline\theta_{t-1}+(1-\mu)\theta_t$ with $\overline\theta_0=\theta_0$. Then $\overline\theta_t$ follows $\mathbf g$ with a lag of $\mu/(1-\mu)$ steps, and in the steady state the root mean square of its period-$S$ oscillation over one period is at most $\rho$ times that of $\theta_t$.
\end{restate}

\begin{proof}\textbf{Step 1 (decomposition).} $\theta_t=\theta_0+t\,\mathbf g+\mathbf p_t$ with $\mathbf p_t=\sum_{j\le t}\mathbf e_{s(j)}$. A full pass adds $\sum_s\mathbf e_s=\mathbf 0$, so $\mathbf p_{t+S}=\mathbf p_t$.

\textbf{Step 2 (linearity).} Write $F$ for the map from a sequence $a_t$ to $\bar a_t=\mu\bar a_{t-1}+(1-\mu)a_t$ with $\bar a_0=a_0$. It is linear and leaves constants unchanged, so $\overline\theta_t=\theta_0+F[t\mathbf g]_t+F[\mathbf p]_t$.

\textbf{Step 3 (the shared part).} By induction on $t$, $F[t]_t=t-\frac{\mu}{1-\mu}(1-\mu^t)$. The shared part therefore passes with a lag that tends to $\mu/(1-\mu)$ steps.

\textbf{Step 4 (the oscillation).} Expand $\mathbf p_t=\mathbf m+\sum_{k=1}^{S-1}\mathbf a_k e^{\mathrm i2\pi kt/S}$. Substituting $e^{\mathrm i\omega t}$ into the recursion shows that $F$ multiplies it by $H(\omega)=(1-\mu)/(1-\mu e^{-\mathrm i\omega})$, up to a transient that decays as $\mu^t$. Since $H(0)=1$, the steady state is $F[\mathbf p]_t=\mathbf m+\sum_kH(2\pi k/S)\,\mathbf a_ke^{\mathrm i2\pi kt/S}$.

\textbf{Step 5 (the bound).} For $1\le k\le S-1$:
\[
|H(2\pi k/S)|^2=\frac{(1-\mu)^2}{1-2\mu\cos(2\pi k/S)+\mu^2}=\frac{(1-\mu)^2}{(1-\mu)^2+4\mu\sin^2(\pi k/S)}\le\rho^2,
\]
because $\sin^2(\pi k/S)\ge\sin^2(\pi/S)$. Parseval's identity over one period gives:
\[
\frac1S\sum_{t=1}^{S}\|F[\mathbf p]_t-\mathbf m\|_2^2=\sum_k|H(2\pi k/S)|^2\|\mathbf a_k\|_2^2\le\rho^2\,\frac1S\sum_{t=1}^{S}\|\mathbf p_t-\mathbf m\|_2^2 .
\]

\textbf{Step 6 (relation targets).} If the embeddings are $L$-Lipschitz in the parameters and their rows have norm at least $c$, Lemma~\ref{flo:lem:stable} bounds the $\operatorname{KL}$ divergence between the targets at two parameter values by $8L/(\tau c)$ times their distance. The oscillation of the targets is therefore at most $8L/(\tau c)$ times the oscillation of the parameters that produce them.
\end{proof}

\textbf{Remarks.} The bound is attained when the oscillation has a single frequency, as for $S=2$. For an arbitrary update sequence, $F$ still multiplies every frequency component by $H(\omega)$, so components at the domain frequencies $2\pi k/S$ are attenuated at least by $\rho$. A per-step form also holds: $\overline\theta_t-\overline\theta_{t-1}=(1-\mu)(\theta_t-\overline\theta_{t-1})$, so $\operatorname{KL}(q(\overline\theta_t)\|q(\overline\theta_{t-1}))\le\frac{8L}{\tau c}(1-\mu)\|\theta_t-\overline\theta_{t-1}\|_2$. Keeping $\rho$ fixed as $S$ grows requires $1-\mu$ to shrink roughly in proportion to $1/S$.

\subsection{Relation Change and the Effectiveness of PARR}
\label{app:relation-change}
\textbf{Propagation flattening on real graphs.} We first verify that the flattening characterized by Theorem~\ref{flo:thm:prop} appears under the candidate construction of FLO, which also contains random-walk and uniformly sampled candidates. For each node, we compute its relation distribution over a candidate set of FLO (Eq.~\eqref{flo:eq:relations}, $\tau=0.2$, $\mathbf W=\mathbf I$) from $\mathbf X$, $\widehat{\mathbf A}\mathbf X$, and $\widehat{\mathbf A}^2\mathbf X$, and divide its KL divergence to the uniform distribution by the logarithm of the number of valid candidate slots, averaged over five draws of the candidate sets. Without any training, one propagation step reduces this divergence for at least 99.5\% of the nodes on every graph, and a second step reduces it further for at least 99.9\%. After two steps, the median divergence decreases by a factor of 2.0--2.3 on Cora, CiteSeer, PubMed, and CS, and by factors of 16 and 19 on Photo and Computers.

\textbf{Relation change and the gain of PARR.} We then test whether this relation change predicts when PARR helps. In the proof of Theorem~\ref{flo:thm:recip}(iii), $R$ is bounded by $\mathbb E\,\operatorname{KL}(P\|Q_B)$, the divergence between relations computed with and without propagation. We measure this quantity on the six graphs of our experiments. For each graph, we take the preprocessed features $\mathbf X$ of Eq.~\eqref{eq:frontend} with $\mathbf W=\mathbf I$, sample 4,096 anchor nodes and their candidate sets as in FLO (the anchor, three random walks of length three, and 16 uniformly sampled nodes), and compute relation distributions (Eq.~\eqref{flo:eq:relations}, $\tau=0.2$) from $\widehat{\mathbf A}^{K}\mathbf X$ and from $\mathbf X$. Table~\ref{tab:relation-change} reports the average KL divergence for $K=1,2$ together with the accuracy drop on each target graph when both relation learning objectives are removed (Table~\ref{tab:ablation}). Propagation changes the relations most on CiteSeer and Cora, for both $K$, and these are the two target graphs on which relation learning contributes most; on the other four graphs, both the relation change and the drop are small (Figure~\ref{fig:relation-change-gain}). Across the six graphs, the Pearson correlation between the relation change and the drop is 0.89 for $K=1$ and 0.81 for $K=2$, and it stays between 0.86 and 0.92 for $K=1$ when any one graph is removed. Since $Q_B$ is a function of $\xi$, Theorem~\ref{flo:thm:recip}(i) gives $\mathbb E\,\operatorname{KL}(P\|Q_B)=R+\mathbb E\,\operatorname{KL}(\overline P\|Q_B)\geq R$, so the measured change bounds $R$ from above and also contains the difference between the attribute-predictable relation and the relation computed from the input features; it is a graph-level measurement rather than an estimate of $R_s$. PARR therefore helps most on the graphs where propagation changes the relations most, which links its gain to the propagation-induced part $R$ analyzed in Theorems~\ref{flo:thm:prop} and~\ref{flo:thm:recip}.

\begin{table}[htbp]
\centering\footnotesize
\caption{Relation change caused by propagation, $\mathbb E\,\operatorname{KL}(P\|Q_B)$, and the accuracy drop (\%) when relation learning is removed.}
\label{tab:relation-change}
\begin{tabular*}{\linewidth}{@{\extracolsep{\fill}}lcccccc}
\toprule
 & Cora & CiteSeer & PubMed & CS & Photo & Computers\\
\midrule
Relation change, $K=1$ & 0.70 & \textbf{0.77} & 0.41 & 0.44 & 0.53 & 0.53\\
Relation change, $K=2$ & \textbf{1.10} & 1.09 & 0.83 & 0.67 & 0.61 & 0.61\\
Drop without $\mathcal L_{\rm kd}$ & 2.57 & \textbf{5.98} & 0.11 & 0.49 & 0.28 & 0.09\\
\bottomrule
\end{tabular*}
\end{table}

\begin{figure}[htbp]
\centering
\includegraphics[width=\linewidth]{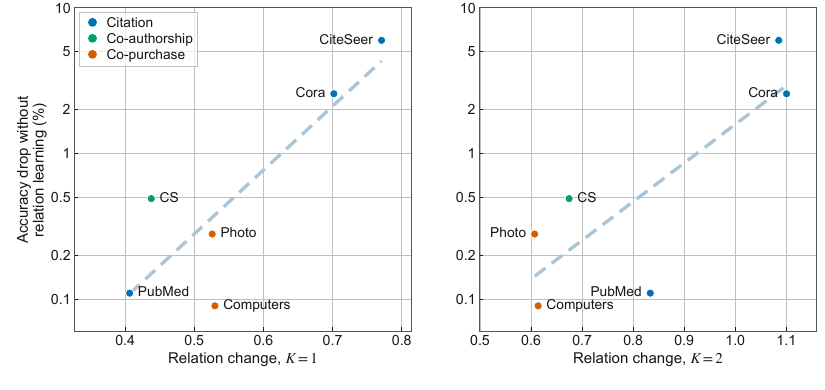}
\caption{Relation change caused by propagation and the accuracy drop (\%) when relation learning is removed, for each target graph. The values are those of Table~\ref{tab:relation-change}. The vertical axis is logarithmic, and the dashed lines are least-squares fits of the logarithm of the drop.}
\label{fig:relation-change-gain}
\end{figure}

\subsection{Relation Invariance}
\label{flo:app:invariance}
The two branches exchange relation distributions instead of embeddings (Section~\ref{flo:sec:parr}). This works because, for nonzero embedding rows, cosine relation distributions are invariant to a common orthogonal change of basis and to independent positive row scaling:
\[
p_{\mathbf D_\nu\mathbf Z\mathbf R}=p_{\mathbf Z},
\qquad
\mathbf R^\top\mathbf R=\mathbf I,\quad
\mathbf D_\nu=\operatorname{diag}(\nu_1,\ldots,\nu_n),\quad\nu_i>0.
\]
Both operations leave every cosine in Eq.~\eqref{flo:eq:relations} unchanged, since $\cos(\nu_i\mathbf R^\top\mathbf z_i,\nu_c\mathbf R^\top\mathbf z_c)=\cos(\mathbf z_i,\mathbf z_c)$ for $\nu_i,\nu_c>0$. Relation learning across the two branches therefore does not require their embeddings to share coordinates. Each slow copy tracks its own online branch, and no parameters are averaged between the graph branch and the attribute branch.

\section{Pre-training Algorithm and Implementation}
\label{app:implementation}
This appendix gives the pre-training procedure of FLO (Algorithm~\ref{alg:flo}), its hyperparameters (Table~\ref{tab:hyperparameters}), the details of candidate sampling and input masking, and its computational cost (Table~\ref{tab:efficiency}).

\begin{algorithm}[H]
\caption{Multi-domain pre-training of FLO}
\label{alg:flo}
\small
\textbf{Input:} source graphs $\{(\mathbf A_s,\mathbf X_s^{\rm raw})\}_{s=1}^{S}$, epochs $E$, warm-up length $T_w$, relation weight $\lambda$, momentum $\mu$, temperature $\tau$.\\
\textbf{Output:} graph encoder $f_G(\cdot\,;\theta_G)$.
\begin{enumerate}[label=\arabic*:,leftmargin=2.2em,itemsep=1pt,topsep=3pt]
\item $\mathbf X_s\leftarrow$ Eq.~\eqref{eq:frontend} for $s=1,\ldots,S$; initialize $\theta_G,\theta_M,D_G,D_M$; $\overline\theta_a\leftarrow\theta_a$ for $a\in\{G,M\}$
\item \textbf{for} $e=1,\ldots,E$ \textbf{do}
\item \hspace{1em}$r_e\leftarrow\min(1,e/T_w)$
\item \hspace{1em}\textbf{for} $s=1,\ldots,S$ \textbf{do}
\item \hspace{2em}draw masks $\mathbf M_G,\mathbf M_M$; $\mathbf Z_G,\mathbf Z_M\leftarrow$ Eq.~\eqref{flo:eq:encoding}; $\mathcal L_{\rm rec}^G,\mathcal L_{\rm rec}^M\leftarrow$ Eq.~\eqref{flo:eq:recon}
\item \hspace{2em}sample anchors $\mathcal I$ and candidate sets $\{\mathcal C_i\}_{i\in\mathcal I}$
\item \hspace{2em}$p_G,p_M\leftarrow$ Eq.~\eqref{flo:eq:relations} on $\mathbf Z_G,\mathbf Z_M$; $q_G,q_M\leftarrow$ Eq.~\eqref{flo:eq:relations} on the slow copies, without gradients
\item \hspace{2em}$\mathcal L\leftarrow$ Eq.~\eqref{flo:eq:objective} with $\mathcal L_{\rm kd}^G,\mathcal L_{\rm kd}^M$ from Eq.~\eqref{flo:eq:bidirectional}
\item \hspace{2em}update $\theta_G,\theta_M,D_G,D_M$ by one AdamW step on $\mathcal L$
\item \hspace{2em}$\overline\theta_a\leftarrow\mu\,\overline\theta_a+(1-\mu)\,\theta_a$ for $a\in\{G,M\}$ \hfill Eq.~\eqref{flo:eq:ema-update}
\item \textbf{return} $f_G(\cdot\,;\theta_G)$
\end{enumerate}
\end{algorithm}

\textbf{Algorithm.} Each epoch visits the $S$ source graphs in order and takes one update on each. The masks of the two branches are drawn independently. Both online branches and both slow copies use the same candidate sets, and each slow copy takes the same masked input as its online branch. After pre-training, the graph encoder is frozen and applied to the target graph without masks, and only the downstream readout is fitted.

\begin{table}[htbp]
\centering\footnotesize
\caption{Hyperparameters of FLO. The pre-training hyperparameters are shared across all target graphs; the last two rows apply only to few-shot node classification.}
\label{tab:hyperparameters}
\begin{tabular*}{\linewidth}{@{\extracolsep{\fill}}lll}\toprule
Component & Parameter & Value\\\midrule
Input & Common dimension $d$ & 1,024\\
Both branches & Hidden width; output width & 512; 1,024\\
Graph branch & Architecture & GCN--PReLU--GCN\\
Attribute branch & Architecture & Linear--GELU--Linear\\
Decoders & Architecture & Separate linear maps\\
Mask & Elementwise keep probability & 0.40\\
Reconstruction & Coefficient $\gamma$ for mean squared row distance & $1/256$\\
Candidates & Self + walk visits + global draws & $1+3\times3+16=26$\\
Anchors & Maximum per source step $N$ & 4,096\\
Relations & Temperature $\tau$; weight $\lambda$ & 0.2; 0.5\\
Warm-up & $T_w$ & 100 epochs\\
EMA & Momentum $\mu$; update frequency & 0.99; each source step\\
Optimizer & Algorithm; learning rate; weight decay & AdamW; $10^{-3}$; $10^{-4}$\\
Training & Epochs; pre-training seed & 3,000; 0\\\addlinespace
Few-shot & Singular components kept for the target graph & \begin{tabular}[t]{@{}l@{}}Cora 16, CiteSeer 8, PubMed 8,\\ CS 256, Photo 64, Computers 512\end{tabular}\\
 & Propagation steps with $\widehat{\mathbf A}$ & Cora 5, others 1\\\bottomrule
\end{tabular*}
\end{table}

\textbf{Candidate sampling and masks.} Three random walks each contribute every visited node over three steps. Sixteen additional nodes are drawn uniformly with replacement. Repeated nodes occupy separate slots; a global draw may equal the anchor or a neighbor. A walk without a valid next neighbor contributes an invalid slot, excluded from the softmax. Input masking samples a full $n_s\times d$ Bernoulli array independently for each branch; columns are not shared. The masks are not divided by their keep probability. The self slot remains valid. Row normalization uses an epsilon floor for zero-norm inputs. When only one direction of relation learning is used in the ablation, its weight is doubled, so that the total weight of relation learning is the same as in FLO.

\textbf{GCN attribute branch.} This variant in Table~\ref{tab:ablation} uses a two-layer GCN with the widths and activation of the graph encoder as the attribute branch; everything else is as in FLO.

\textbf{Computational cost.} In the main configuration, input and output widths both equal $d=1024$. With hidden width $h$ and $|\mathcal C_i|=26$ candidate slots, a source step costs $O(n_sdh+n_sd^2+|\mathcal T_s|(d+h))$ for the online branches, their slow copies, and reconstruction, plus $O(|\mathcal I||\mathcal C_i|d)$ for the relation distributions. The singular value decomposition is a one-time preprocessing cost. Transfer requires only the two-layer online GNN.

Table~\ref{tab:efficiency} reports the measured cost of all methods. The pre-training time is the median over 20 epochs, and the forward time is the median over 20 passes after 3 warm-up passes. FLO pre-trains about as fast as DGI and MDGPT, at 0.14--0.19\,s per epoch, and faster than the other multi-domain graph pre-training methods. Its peak memory of 5.6\,GB is the second lowest among all methods and less than a third of that of MDGPT, MDGFM, and BRIDGE. The extra cost of FLO lies in pre-training: the attribute branch and the two decoders raise its trainable parameters to 4.20M. At inference, FLO keeps only the graph encoder, a two-layer GNN with 1.05M parameters, a quarter of this number. Its forward pass is faster than those of MDGPT, SAMGPT, MDGFM, and BRIDGE on every target graph.
\begin{table}[!t]
\centering
\footnotesize
\setlength{\tabcolsep}{2.5pt}
\caption{Computational cost. Top: pre-training time per epoch (s) with each graph held out as the target, peak GPU memory during pre-training, and trainable parameters. Bottom: forward time (ms) of the frozen encoder on each target graph and parameters at inference. All methods are measured on the same NVIDIA RTX 5090 GPU.}
\label{tab:efficiency}
\begin{tabular*}{\linewidth}{@{\extracolsep{\fill}}lcccccccc}
\toprule
Method & Cora & CiteSeer & PubMed & CS & Photo & Computers & Memory (GB) & Params (M) \\
\midrule
\multicolumn{9}{c}{\textbf{Pre-training time per epoch (s)}} \\
\midrule
DGI & 0.189 & 0.189 & 0.145 & 0.148 & 0.177 & 0.159 & 19.6 & 0.39 \\
GRACE & 0.747 & 1.063 & 0.532 & 0.560 & 0.813 & 0.739 & 12.9 & 0.39 \\
BGRL & 0.102 & 0.102 & 0.099 & 0.097 & 0.076 & 0.064 & 3.5 & 0.53 \\
FUG & 0.167 & 0.165 & 0.152 & 0.133 & 0.138 & 0.110 & 10.6 & 10.49 \\
TIG & 0.231 & 0.230 & 0.209 & 0.201 & 0.183 & 0.152 & 14.7 & 3.81 \\
GFT & 15.30 & 15.15 & 11.98 & 10.90 & 11.88 & 8.59 & 14.2 & 12.80 \\
MDGPT & 0.179 & 0.165 & 0.126 & 0.128 & 0.153 & 0.142 & 25.2 & 0.15 \\
SAMGPT & 0.345 & 0.336 & 0.281 & 0.280 & 0.325 & 0.282 & 7.1 & 0.28 \\
MDGFM & 1.871 & 1.854 & 1.156 & 1.231 & 1.765 & 1.512 & 20.0 & 0.15 \\
BRIDGE & 0.784 & 0.809 & 0.586 & 0.600 & 0.722 & 0.661 & 27.2 & 0.15 \\
\midrule
\textbf{FLO} & 0.188 & 0.178 & 0.150 & 0.147 & 0.169 & 0.136 & 5.6 & 4.20 \\
\midrule
\multicolumn{9}{c}{\textbf{Inference time per forward pass (ms)}} \\
\midrule
DGI & 1.29 & 1.35 & 3.18 & 3.14 & 1.98 & 2.99 & \NA & 0.13 \\
GRACE & 1.12 & 1.18 & 3.33 & 3.87 & 3.30 & 6.03 & \NA & 0.26 \\
BGRL & 1.42 & 1.45 & 3.58 & 4.05 & 3.49 & 6.22 & \NA & 0.26 \\
FUG & 3.54 & 3.86 & 18.56 & 19.95 & 13.74 & 26.29 & \NA & 10.49 \\
TIG & 1.36 & 2.25 & 8.53 & 9.88 & 7.63 & 15.71 & \NA & 1.71 \\
GFT & 1.34 & 1.46 & 10.96 & 13.00 & 10.63 & 21.57 & \NA & 2.36 \\
MDGPT & 4.30 & 4.33 & 20.45 & 31.28 & 34.88 & 74.03 & \NA & 0.14 \\
SAMGPT & 4.34 & 4.36 & 20.61 & 31.55 & 35.40 & 74.67 & \NA & 0.14 \\
MDGFM & 4.40 & 4.39 & 20.66 & 32.00 & 46.15 & 95.42 & \NA & 0.14 \\
BRIDGE & 4.29 & 4.29 & 20.40 & 41.13 & 46.52 & 95.48 & \NA & 0.14 \\
\midrule
\textbf{FLO} & 2.61 & 2.80 & 13.54 & 12.04 & 7.63 & 15.10 & \NA & 1.05 \\
\bottomrule
\end{tabular*}
\end{table}

\section{Related Work}
\label{app:relatedwork}

\paragraph{Graph pre-training and graph foundation models.}
Graph pre-training trains GNNs with self-supervised objectives, which are either contrastive \citep{veličković2018deep,you2020graph,zhu2020deep,thakoor2022largescale} or generative \citep{hou2022graphmae}.
Early studies transfer structural patterns across graphs \citep{qiu2020gcc} or pre-train a GNN at both the node and the graph level \citep{hu2020strategies}.
Most of these methods pre-train within a single domain.
Multi-domain graph pre-training has recently attracted growing interest as a route toward GFMs \citep{mao2024position,liu2025graph}.
One line of work represents node attributes as text and unifies them with language models.
OFA casts node, link, and graph classification into a single prompt-graph format \citep{liu2024one}, ZeroG encodes node attributes and class semantics with a language model for zero-shot transfer across datasets \citep{li2024zerog}, and UniGraph pre-trains a cascaded language model and GNN with masked autoencoding \citep{he2025unigraph}.
These methods rely on text to unify graphs from different domains.

\paragraph{Text-free multi-domain graph pre-training.}
For graphs without text, feature unification methods map the node attributes of different domains into a shared space.
FUG captures transferable distributions among attributes with a dimensional encoder \citep{zhao2024fug}, and STAGE encodes statistical dependencies between attributes \citep{pmlr-v267-shen25p}.
PatchNet decomposes node attributes into learnable graph patches \citep{sun2025handling}, and TIG mines feature-dimension relations shared across graphs \citep{zhao2026towards}.
LEDA aligns the node attributes of different domains with a shared latent distribution \citep{shan2026leda}, and ALL-IN builds covariance-based representations in a shared random space \citep{eliasof2026bridging}.
MUG applies universal pre-training to heterogeneous graphs through a dimension-aware encoder \citep{shan2026mug}.
Some methods avoid a shared attribute space: GraphControl pre-trains on graph structure and injects the node attributes of the target graph as conditional inputs \citep{zhu2024graphcontrol}, and GraphAny solves linear graph models in closed form on each new graph and combines them with attention \citep{zhao2025fully}.
Other methods add learnable alignment parameters to the model.
GCOPE connects the source graphs through domain-specific virtual nodes \citep{zhao2024all}, MDGPT learns domain tokens \citep{yu2024text}, and SAMGPT adds structure tokens \citep{yu2025samgpt}.
BRIDGE learns domain-invariant feature aligners \citep{yuan2025much}, and MF-GIA derives domain-specific transformations from gradient fingerprints \citep{zhuo2026modality}.
Another line seeks structural knowledge shared across domains.
GFT treats computation trees as a transferable vocabulary \citep{wang2024gft}, and RiemannGFM learns a vocabulary of trees and cycles in Riemannian space \citep{sun2025riemanngfm}.
MDGFM aligns topologies across domains \citep{wang2025multidomain}, and AnyGraph routes graphs to a mixture of graph experts \citep{xia2026anygraph}.
GraphGlue merges graphs from different domains into one smooth Riemannian manifold \citep{sun2026multi}, SPG learns a shared prototype geometry that guides propagation \citep{yang2026graph}, and ProGFM learns propagation patterns shared across edges and feature dimensions \citep{wang2026beyond}.
Beyond node classification, GFT and GIT pre-train a single model for node-, edge-, and graph-level tasks \citep{wang2024gft,pmlr-v267-wang25eq}, and holographic node representations serve tasks of any order after pre-training \citep{bevilacqua2025holographic}.
Most of these methods align domains through the input features or through parameters trained jointly with a message-passing encoder.

\paragraph{Relational knowledge distillation.}
Knowledge distillation trains a student to match the softened outputs of a teacher \citep{hinton2015distilling}.
Relational knowledge distillation transfers the relations among samples instead of the outputs of individual samples.
RKD matches pairwise distances and triplet angles \citep{park2019relational}.
LSP matches the local structure of each node between a teacher and a student graph convolutional network \citep{yang2020distilling}.
SEED and ReSSL match the similarity distribution of each sample over a set of other samples \citep{fang2021seed,zheng2021ressl}.
Deep mutual learning removes the fixed teacher and lets a cohort of networks learn from the predictions of each other \citep{zhang2018deep}.
FreeKD lets two shallow GNNs exchange knowledge in directions chosen by reinforcement learning \citep{feng2022freekd}.
On graphs, GLNN distills a trained GNN into a multilayer perceptron that needs no message passing at inference \citep{zhang2022graphless}.
NOSMOG further distills the similarities between node representations \citep{tian2023learning}, and TINED injects the parameters of the teacher and distills its Dirichlet energy \citep{zhou2025tined}.

\paragraph{Slowly updated targets.}
Averaging the iterates of stochastic optimization accelerates its convergence \citep{polyak1992acceleration}.
Momentum teachers keep an exponential moving average of the parameters of a student network and use it to produce targets, as in Mean Teacher \citep{tarvainen2017mean}, MoCo \citep{he2020momentum}, and BYOL \citep{grill2020bootstrap}.
BGRL applies this design to graph self-supervised learning \citep{thakoor2022largescale}, and ReSSL uses it to produce relation targets \citep{zheng2021ressl}.
Recent studies scale the momentum of the average with the batch size \citep{busbridge2023scale} and show that the averaged weights generalize better and are more robust to label noise \citep{morales-brotons2024exponential}.
Averaging the weights along the training trajectory improves generalization to unseen domains \citep{cha2021swad}.
Fast and slow weights also appear in optimization and continual learning.
Lookahead moves a set of slow weights toward the fast weights every $k$ steps \citep{zhang2019lookahead}, and CLS-ER keeps short-term and long-term semantic memories of a working model, following complementary learning systems theory \citep{mcclelland1995there,arani2022learning}.

\section{Experimental Details and Additional Results}\label{app:experiments}
This appendix describes the datasets, baselines, and evaluation protocols of Section~\ref{sec:experiments}. It then reports the few-shot results on every target graph, the computation of the loss curves in Figure~\ref{fig:source-count-six}(b), the effect of the pre-training budget, and the sensitivity of FLO to its hyperparameters.

\subsection{Datasets}\label{app:datasets}
The statistics of all datasets are summarized in Table~\ref{tab:datasets}. Cora, CiteSeer, and PubMed are citation networks, where nodes are papers, edges are citations, and node features are bag-of-words vectors. CS is a co-authorship network, where nodes are authors and edges indicate co-authored papers. Photo and Computers are co-purchase networks from Amazon, where nodes are products and edges indicate that two products are frequently bought together. In the leave-one-domain-out setting, each graph is held out in turn as the target graph and the other five graphs are used for pre-training. For link prediction we use the co-authorship network Physics, the Wikipedia hyperlink network WikiCS, and the citation network DBLP, and for graph classification we use COIL-RAG, ENZYMES, and PROTEINS from the TUDataset collection; for both tasks, the encoder is pre-trained on all six node classification datasets.

All datasets are loaded with PyTorch Geometric: Cora, CiteSeer, and PubMed through \texttt{Planetoid}; Photo and Computers through \texttt{Amazon}; CS and Physics through \texttt{Coauthor}; WikiCS through \texttt{WikiCS}; DBLP through \texttt{CitationFull}; and COIL-RAG, ENZYMES, and PROTEINS through \texttt{TUDataset}.

\begin{table}[htbp]
\centering
\footnotesize
\setlength{\tabcolsep}{4pt}
\caption{Statistics of the datasets. Edges are counted once per undirected pair; for graph classification, the numbers of nodes and edges are averaged over the graphs of each dataset.}
\label{tab:datasets}
\begin{tabular*}{\linewidth}{@{\extracolsep{\fill}}lrrrrrl}
\toprule
Dataset & Graphs & Nodes & Edges & Features & Classes & Domain \\
\midrule
\multicolumn{7}{c}{\textbf{Node classification}} \\
\midrule
Cora & 1 & 2,708 & 5,278 & 1,433 & 7 & Citation \\
CiteSeer & 1 & 3,327 & 4,552 & 3,703 & 6 & Citation \\
PubMed & 1 & 19,717 & 44,324 & 500 & 3 & Citation \\
CS & 1 & 18,333 & 81,894 & 6,805 & 15 & Co-authorship \\
Photo & 1 & 7,650 & 119,081 & 745 & 8 & Co-purchase \\
Computers & 1 & 13,752 & 245,861 & 767 & 10 & Co-purchase \\
\midrule
\multicolumn{7}{c}{\textbf{Link prediction}} \\
\midrule
Physics & 1 & 34,493 & 247,962 & 8,415 & -- & Co-authorship \\
WikiCS & 1 & 11,701 & 215,603 & 300 & -- & Web \\
DBLP & 1 & 17,716 & 52,867 & 1,639 & -- & Citation \\
\midrule
\multicolumn{7}{c}{\textbf{Graph classification}} \\
\midrule
COIL-RAG & 3,900 & 3.01 & 3.02 & 64 & 100 & Computer vision \\
ENZYMES & 600 & 32.63 & 62.14 & 21 & 6 & Bioinformatics \\
PROTEINS & 1,113 & 39.06 & 72.82 & 32 & 2 & Bioinformatics \\
\bottomrule
\end{tabular*}
\end{table}

\subsection{Baselines}\label{app:baselines}
\textbf{Graph self-supervised learning methods.} DGI \citep{veličković2018deep} maximizes the mutual information between node representations and a summary vector of the graph, and takes the representations of a graph with shuffled node features as negative samples. GRACE \citep{zhu2020deep} builds two views of a graph by removing edges and masking features, and contrasts the representations of each node across the two views. BGRL \citep{thakoor2022largescale} predicts the representations of one augmented view from another view and takes its targets from an encoder updated by an exponential moving average, so it needs no negative samples. Since these methods require a common input dimension, we reduce the features of every graph to a common dimension by singular value decomposition and pre-train each method on the source graphs for 1{,}000 epochs.

\textbf{Feature unification methods.} FUG \citep{zhao2024fug} designs a dimension encoder with contrastive constraints that emulates the basis transformation of principal component analysis. The encoder maps graphs with different feature dimensions into a shared space without reshaping their features, and a global uniformity constraint replaces negative sampling in graph contrastive pre-training. TIG \citep{zhao2026towards} mines transfer-invariant metadata from node features, namely feature-dimension relations shared across graphs. It turns features of different dimensions into a unified structural representation and learns graph knowledge from it by self-supervised learning.

\textbf{Multi-domain graph pre-training methods.} GFT \citep{wang2024gft} treats computation trees as transferable patterns. It learns a tree vocabulary by reconstructing computation trees during pre-training and casts downstream tasks as tree classification. MDGPT \citep{yu2024text} aligns the features of source domains with learnable domain tokens and adapts to the target domain with a unifying prompt and a mixing prompt. SAMGPT \citep{yu2025samgpt} further introduces structure tokens that modulate message passing in each domain to align structural patterns, and adapts with holistic and specific prompts. MDGFM \citep{wang2025multidomain} aligns topologies across domains. It refines each graph with a structure learned from node features and propagated features, and pre-trains by contrasting the original and refined graphs. BRIDGE \citep{yuan2025much} aligns multi-domain features with domain-invariant aligners during pre-training. It transfers source knowledge to the target through a lightweight mixture of experts and bounds the generalization error of this transfer.

All baselines are run with their official code (Table~\ref{tab:baseline-code}) and the settings recommended in their papers. Following their original settings, MDGPT, SAMGPT, MDGFM, and BRIDGE reduce the feature dimension of each graph to 50 with principal component analysis, and GFT obtains 768-dimensional features with singular value decomposition. Since TIG builds a node--feature graph whose edge weights are the feature values, the features of the cross-task datasets that take negative values or values above one (WikiCS, ENZYMES, and PROTEINS) are min--max scaled to $[0,1]$ in each dimension before its preprocessing. For all methods, the representations produced by the pre-trained encoder are used for downstream evaluation.

\begin{table}[htbp]
\centering
\footnotesize
\caption{Code of the baselines.}
\label{tab:baseline-code}
\begin{tabular*}{\linewidth}{@{\extracolsep{\fill}}llll@{}}
\toprule
Method & Code & Method & Code\\
\midrule
DGI & \href{https://github.com/PetarV-/DGI}{github.com/PetarV-/DGI} & GFT & \href{https://github.com/Zehong-Wang/GFT}{github.com/Zehong-Wang/GFT}\\
GRACE & \href{https://github.com/CRIPAC-DIG/GRACE}{github.com/CRIPAC-DIG/GRACE} & MDGPT & \href{https://anonymous.4open.science/r/MDGPT}{anonymous.4open.science/r/MDGPT}\\
BGRL & \href{https://github.com/nerdslab/bgrl}{github.com/nerdslab/bgrl} & SAMGPT & \href{https://github.com/blue-soda/SAMGPT}{github.com/blue-soda/SAMGPT}\\
FUG & \href{https://github.com/hedongxiao-tju/FUG}{github.com/hedongxiao-tju/FUG} & MDGFM & \href{https://github.com/wbkzwqtzw/MDGFM}{github.com/wbkzwqtzw/MDGFM}\\
TIG & \href{https://github.com/hedongxiao-tju/TIG}{github.com/hedongxiao-tju/TIG} & BRIDGE & \href{https://github.com/RingBDStack/BRIDGE}{github.com/RingBDStack/BRIDGE}\\
\bottomrule
\end{tabular*}
\end{table}

\subsection{Evaluation Protocols}\label{app:probe}
Since each prior work covers only part of our evaluation scenarios, we follow the corresponding prior protocol in each scenario, and all encoders are frozen during downstream evaluation. For node classification, following FUG \citep{zhao2024fug}, we train a logistic regression classifier on the $\ell_2$-normalized node representations. The nodes of the target graph are split into 20\% and 80\%; the regularization strength of the classifier is selected on five random 50\%/50\% training/validation splits of the 20\%, which gives a 10\%/10\%/80\% split, and the accuracy on the remaining 80\% is averaged over three random splits. For few-shot node classification, following TIG \citep{zhao2026towards}, we use the same 500 sampled $k$-shot tasks of each target graph for all methods and train a logistic regression classifier on the $k$ labeled nodes of each class. For link prediction, the representations of each method are standardized column-wise, and an edge classifier with the same architecture is trained for all methods with early stopping on the validation edges; on Physics, WikiCS, and DBLP we use five random 85\%/5\%/10\% splits of the edges; message passing uses only the training edges. For graph classification, the node representations of a graph are pooled by concatenating mean and max pooling, and the same two-layer classifier is trained on 20\% of the graphs and tested on the remaining 80\%, over 50 stratified splits. The pre-training hyperparameters of FLO are shared across all target graphs (Table~\ref{tab:hyperparameters}), and FLO is evaluated at the final epoch of pre-training. The random seed is fixed to 0 for reproducibility.

\textbf{Environment.} All experiments run on one Linux server, and all methods are timed on the same type of GPU (Table~\ref{tab:environment}).

\begin{table}[htbp]
\centering
\footnotesize
\caption{Experimental environment.}
\label{tab:environment}
\begin{tabular*}{\linewidth}{@{\extracolsep{\fill}}ll@{}}
\toprule
Component & Specification\\
\midrule
Operating system & Ubuntu 22.04.5 LTS, Linux kernel 6.8.0\\
CPU & Intel Xeon Silver 4410Y, 48 threads\\
Memory & 128 GB\\
GPU & NVIDIA GeForce RTX 5090, 32 GB\\
Software & Python 3.10.12, PyTorch 2.11.0 with CUDA 12.8, PyTorch Geometric 2.8.0\\
 & NumPy 2.2.6, scikit-learn 1.7.2\\
\bottomrule
\end{tabular*}
\end{table}

\subsection{Few-Shot Node Classification}\label{app:fewshot}
This subsection details the few-shot evaluation of Section~\ref{sec:q1} and reports the results on every target graph. The pre-trained encoder of FLO is the same as in node classification. In the few-shot setting, the number of singular components kept for the target graph in Eq.~\eqref{eq:frontend} is set for each target graph (Table~\ref{tab:hyperparameters}), and the remaining dimensions are padded with zeros. Following GGD \citep{zheng2022rethinking}, we further enhance the similarity relations between connected nodes by propagating the frozen representations over the target graph with $\widehat{\mathbf A}$ for a number of steps set for each target graph (Table~\ref{tab:hyperparameters}).
 Tables~\ref{tab:fewshot-1}, \ref{tab:fewshot-3}, and~\ref{tab:fewshot-5} report the results of all methods on each target graph; the results on Cora, CiteSeer, and CS are plotted in Figure~\ref{fig:fewshot}. FLO achieves the best average accuracy in all three settings, and it ranks first on four, three, and three of the six target graphs with one, three, and five labeled nodes per class, respectively.

\begin{table}[tb]
\centering
\footnotesize
\setlength{\tabcolsep}{1.5pt}
\caption{Accuracy (\%) of 1-shot node classification with standard deviations over the 500 tasks shared by all methods. The best and second-best results are in bold and underlined.}\label{tab:fewshot-1}
\begin{tabular*}{\linewidth}{@{\extracolsep{\fill}}lccccccc}
\toprule
Method & Cora & CiteSeer & PubMed & CS & Photo & Computers & Avg. \\
\midrule
DGI & 33.09$_{\pm5.51}$ & 33.79$_{\pm5.60}$ & 47.10$_{\pm6.45}$ & 71.82$_{\pm6.36}$ & 56.42$_{\pm8.15}$ & 42.48$_{\pm9.90}$ & 47.45 \\
GRACE & 28.13$_{\pm4.46}$ & 24.20$_{\pm3.88}$ & 39.81$_{\pm5.61}$ & 47.23$_{\pm5.77}$ & 42.35$_{\pm7.61}$ & 38.77$_{\pm7.31}$ & 36.75 \\
BGRL & 34.19$_{\pm4.42}$ & 25.04$_{\pm3.75}$ & 43.37$_{\pm5.66}$ & 57.58$_{\pm5.84}$ & 60.32$_{\pm10.14}$ & 50.95$_{\pm11.17}$ & 45.24 \\
FUG & 46.40$_{\pm8.37}$ & 36.07$_{\pm6.81}$ & 52.23$_{\pm9.47}$ & \textbf{78.02}$_{\pm5.88}$ & \underline{63.62}$_{\pm9.77}$ & \underline{51.27}$_{\pm12.82}$ & 54.60 \\
TIG & \underline{48.88}$_{\pm8.80}$ & 44.71$_{\pm9.02}$ & 50.94$_{\pm10.00}$ & 75.28$_{\pm6.22}$ & 63.52$_{\pm10.47}$ & 49.47$_{\pm13.08}$ & \underline{55.47} \\
GFT & 27.25$_{\pm4.78}$ & 26.78$_{\pm4.31}$ & 45.36$_{\pm6.00}$ & 63.49$_{\pm6.42}$ & 20.01$_{\pm4.37}$ & 14.62$_{\pm5.12}$ & 32.92 \\
MDGPT & 42.49$_{\pm7.94}$ & \underline{45.13}$_{\pm8.81}$ & \textbf{55.34}$_{\pm8.32}$ & 72.98$_{\pm6.61}$ & 55.92$_{\pm8.31}$ & 41.05$_{\pm8.44}$ & 52.15 \\
SAMGPT & 44.81$_{\pm8.42}$ & 41.52$_{\pm8.15}$ & 45.42$_{\pm8.66}$ & 76.48$_{\pm6.17}$ & 59.17$_{\pm8.63}$ & 46.73$_{\pm9.64}$ & 52.35 \\
MDGFM & 42.43$_{\pm6.05}$ & 38.84$_{\pm7.54}$ & 48.21$_{\pm7.93}$ & 74.11$_{\pm6.55}$ & 63.43$_{\pm8.91}$ & 48.86$_{\pm9.56}$ & 52.65 \\
BRIDGE & 42.77$_{\pm6.37}$ & 40.70$_{\pm7.61}$ & 53.57$_{\pm7.78}$ & 75.14$_{\pm6.36}$ & 53.23$_{\pm7.21}$ & 38.89$_{\pm7.11}$ & 50.72 \\
\midrule
\textbf{FLO} & \textbf{55.95}$_{\pm8.41}$ & \textbf{45.93}$_{\pm8.99}$ & \underline{55.26}$_{\pm9.59}$ & \underline{77.20}$_{\pm6.24}$ & \textbf{67.76}$_{\pm9.55}$ & \textbf{52.99}$_{\pm10.76}$ & \textbf{59.18} \\
\bottomrule
\end{tabular*}
\end{table}

\begin{table}[tb]
\centering
\footnotesize
\setlength{\tabcolsep}{1.5pt}
\caption{Accuracy (\%) of 3-shot node classification with standard deviations over the 500 tasks shared by all methods. The best and second-best results are in bold and underlined.}\label{tab:fewshot-3}
\begin{tabular*}{\linewidth}{@{\extracolsep{\fill}}lccccccc}
\toprule
Method & Cora & CiteSeer & PubMed & CS & Photo & Computers & Avg. \\
\midrule
DGI & 45.56$_{\pm5.08}$ & 44.96$_{\pm4.42}$ & 55.58$_{\pm5.12}$ & 85.26$_{\pm2.08}$ & 70.17$_{\pm5.23}$ & 56.98$_{\pm7.27}$ & 59.75 \\
GRACE & 38.59$_{\pm4.00}$ & 30.63$_{\pm3.88}$ & 45.68$_{\pm4.85}$ & 65.49$_{\pm3.51}$ & 57.22$_{\pm6.52}$ & 50.64$_{\pm6.08}$ & 48.04 \\
BGRL & 46.74$_{\pm4.09}$ & 31.66$_{\pm3.58}$ & 51.06$_{\pm5.51}$ & 75.12$_{\pm3.29}$ & 72.63$_{\pm7.23}$ & 62.05$_{\pm8.43}$ & 56.54 \\
FUG & \underline{61.63}$_{\pm5.42}$ & 47.62$_{\pm4.94}$ & 60.82$_{\pm6.32}$ & \underline{86.55}$_{\pm1.70}$ & 65.91$_{\pm8.90}$ & 58.85$_{\pm10.17}$ & 63.56 \\
TIG & 58.53$_{\pm6.17}$ & 56.39$_{\pm5.03}$ & 56.48$_{\pm7.81}$ & 83.63$_{\pm3.04}$ & 68.24$_{\pm10.30}$ & 56.31$_{\pm11.03}$ & 63.26 \\
GFT & 37.84$_{\pm4.25}$ & 35.49$_{\pm3.97}$ & 51.90$_{\pm5.65}$ & 78.95$_{\pm2.74}$ & 22.25$_{\pm4.74}$ & 17.49$_{\pm6.30}$ & 40.65 \\
MDGPT & 54.99$_{\pm6.30}$ & \underline{57.34}$_{\pm4.97}$ & \textbf{63.88}$_{\pm6.02}$ & 82.70$_{\pm2.45}$ & 65.47$_{\pm6.54}$ & 53.34$_{\pm7.18}$ & 62.95 \\
SAMGPT & 58.66$_{\pm6.25}$ & 53.81$_{\pm5.20}$ & 52.05$_{\pm7.82}$ & 85.27$_{\pm2.00}$ & 68.73$_{\pm6.40}$ & 60.28$_{\pm7.14}$ & 63.13 \\
MDGFM & 53.68$_{\pm4.89}$ & 50.41$_{\pm4.89}$ & 57.41$_{\pm6.50}$ & 84.93$_{\pm2.17}$ & \textbf{76.69}$_{\pm5.22}$ & \textbf{62.95}$_{\pm6.47}$ & \underline{64.34} \\
BRIDGE & 56.36$_{\pm5.15}$ & 53.38$_{\pm4.94}$ & 62.98$_{\pm6.00}$ & 85.53$_{\pm1.96}$ & 65.52$_{\pm5.55}$ & 51.56$_{\pm6.21}$ & 62.55 \\
\midrule
\textbf{FLO} & \textbf{65.78}$_{\pm5.37}$ & \textbf{57.82}$_{\pm5.37}$ & \underline{63.73}$_{\pm6.42}$ & \textbf{87.86}$_{\pm1.82}$ & \underline{75.74}$_{\pm6.26}$ & \underline{62.90}$_{\pm7.00}$ & \textbf{68.97} \\
\bottomrule
\end{tabular*}
\end{table}

\begin{table}[tb]
\centering
\footnotesize
\setlength{\tabcolsep}{1.5pt}
\caption{Accuracy (\%) of 5-shot node classification with standard deviations over the 500 tasks shared by all methods. The best and second-best results are in bold and underlined.}\label{tab:fewshot-5}
\begin{tabular*}{\linewidth}{@{\extracolsep{\fill}}lccccccc}
\toprule
Method & Cora & CiteSeer & PubMed & CS & Photo & Computers & Avg. \\
\midrule
DGI & 53.07$_{\pm3.79}$ & 50.55$_{\pm3.54}$ & 60.34$_{\pm4.47}$ & 87.77$_{\pm1.20}$ & 74.10$_{\pm4.33}$ & 63.16$_{\pm4.49}$ & 64.83 \\
GRACE & 44.50$_{\pm3.25}$ & 34.72$_{\pm3.42}$ & 49.95$_{\pm4.23}$ & 71.96$_{\pm2.75}$ & 64.24$_{\pm5.42}$ & 56.71$_{\pm5.15}$ & 53.68 \\
BGRL & 53.28$_{\pm3.37}$ & 35.74$_{\pm3.11}$ & 56.07$_{\pm4.68}$ & 81.12$_{\pm2.22}$ & \underline{77.83}$_{\pm5.18}$ & \underline{67.13}$_{\pm6.08}$ & 61.86 \\
FUG & \underline{67.31}$_{\pm3.23}$ & 53.28$_{\pm3.86}$ & 64.51$_{\pm4.64}$ & \underline{88.13}$_{\pm1.13}$ & 66.29$_{\pm7.50}$ & 62.83$_{\pm8.18}$ & 67.06 \\
TIG & 63.70$_{\pm4.59}$ & 60.66$_{\pm3.45}$ & 59.64$_{\pm6.58}$ & 86.62$_{\pm1.98}$ & 69.20$_{\pm8.48}$ & 60.80$_{\pm8.49}$ & 66.77 \\
GFT & 44.13$_{\pm3.73}$ & 40.85$_{\pm3.38}$ & 56.09$_{\pm4.75}$ & 82.69$_{\pm1.62}$ & 24.92$_{\pm4.83}$ & 19.76$_{\pm6.37}$ & 44.74 \\
MDGPT & 60.86$_{\pm3.97}$ & \underline{61.87}$_{\pm3.28}$ & \underline{67.59}$_{\pm4.71}$ & 85.35$_{\pm1.51}$ & 70.45$_{\pm4.98}$ & 60.12$_{\pm5.19}$ & 67.71 \\
SAMGPT & 64.52$_{\pm3.90}$ & 59.02$_{\pm3.43}$ & 55.45$_{\pm7.08}$ & 87.01$_{\pm1.24}$ & 72.49$_{\pm5.25}$ & 66.07$_{\pm4.45}$ & 67.43 \\
MDGFM & 58.91$_{\pm3.70}$ & 55.27$_{\pm3.34}$ & 62.87$_{\pm5.39}$ & 87.28$_{\pm1.30}$ & \textbf{80.65}$_{\pm3.25}$ & \textbf{68.51}$_{\pm3.91}$ & \underline{68.92} \\
BRIDGE & 62.45$_{\pm3.43}$ & 58.70$_{\pm3.29}$ & \textbf{67.64}$_{\pm4.69}$ & 87.68$_{\pm1.28}$ & 71.09$_{\pm4.29}$ & 58.77$_{\pm4.91}$ & 67.72 \\
\midrule
\textbf{FLO} & \textbf{68.38}$_{\pm4.03}$ & \textbf{62.03}$_{\pm3.52}$ & 67.34$_{\pm4.87}$ & \textbf{89.43}$_{\pm1.21}$ & 77.43$_{\pm5.42}$ & 66.44$_{\pm4.61}$ & \textbf{71.84} \\
\bottomrule
\end{tabular*}
\end{table}

\subsection{Pre-training Loss Curves}\label{app:loss-curves}
This subsection describes how the loss curves in Figure~\ref{fig:source-count-six}(b) are computed. Figure~\ref{fig:source-count-six}(b) uses the first 40 epochs of the same leave-one-domain-out runs as the main experiments. For FLO, each point is the per-epoch mean over the five source graphs of the graph encoder's objective, $\mathcal L_{\mathrm{rec}}^G+\lambda r_e\mathcal L_{\mathrm{kd}}^G$; BRIDGE and SAMGPT use their logged pre-training objectives. For each method and target graph separately, we plot $(\log L_e-\min_{t\leq40}\log L_t)/(\max_{t\leq40}\log L_t-\min_{t\leq40}\log L_t)$. Thus the figure compares changes over time rather than the magnitudes of different objectives. The largest one-epoch increases are computed from the unnormalized losses as $\max_{2\leq e\leq40} L_e/L_{e-1}-1$.

\subsection{Effect of Pre-training Epochs}\label{app:epochs}
FLO is evaluated at the final epoch of a fixed budget of 3{,}000 epochs, without selecting checkpoints on the target graph (Section~\ref{sec:setup}). This subsection examines whether the results depend on this budget. Figure~\ref{fig:epochs} shows the accuracy on each target graph at pre-training budgets from 300 to 3{,}000 epochs, with the frozen graph encoder and the evaluation protocol of the main experiments. From 300 to 1{,}000 epochs, the accuracy rises by 1.00 points on CiteSeer and 0.35 points on Cora, the two graphs on which propagation changes the relations most (Appendix~\ref{app:relation-change}), and changes by at most 0.16 points on the other four graphs. Between 1{,}000 and 3{,}000 epochs, the accuracy changes by at most 0.25 points on Cora, PubMed, CS, Photo, and Computers, and by 0.53 points on CiteSeer, where it still increases slowly. FLO therefore needs no target-specific early stopping, and all reported results use the fixed budget of 3{,}000 epochs.

\begin{figure}[htbp]
\centering
\includegraphics[width=\linewidth]{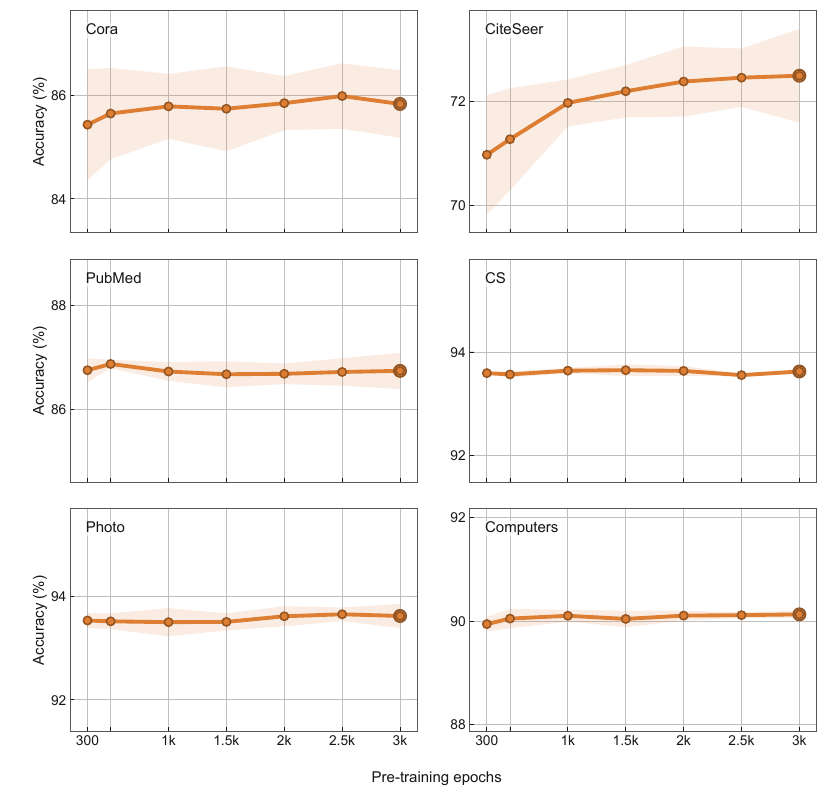}
\caption{Accuracy (\%) of node classification on each target graph at pre-training budgets of 300, 500, 1{,}000, 1{,}500, 2{,}000, 2{,}500, and 3{,}000 epochs. Shaded bands show the standard deviation over the three random splits of node classification, all panels share the same vertical span, and the circled point is the budget used in all other experiments.}
\label{fig:epochs}
\end{figure}

\subsection{Hyperparameter Sensitivity}
We vary the temperature $\tau$, the relation weight $\lambda$, and the momentum $\mu$, each over five values around its default, and keep all other hyperparameters fixed. We evaluate every setting for 2,000 pre-training epochs under each of the six leave-one-domain-out configurations. As shown in Figure~\ref{fig:sensitivity}, the accuracy varies by at most 1.48 points on CiteSeer and by at most 0.86 points on the other graphs, and the default values are within 0.45 points of the best setting on every graph. Removing the warm-up of the relation weight lowers the accuracy by 0.36 points on Cora and 0.04 points on Photo, and extending it to 200 epochs changes the accuracy by at most 0.08 points.

We also vary the learning rate over $\{10^{-4},5\cdot10^{-4},10^{-3},5\cdot10^{-3},10^{-2}\}$ and the hidden width of both branches over $\{128,256,512,1{,}024,2{,}048\}$ with the same budget (Figure~\ref{fig:sensitivity}). From $5\cdot10^{-4}$ to $10^{-2}$, the learning rate changes the accuracy by at most 0.56 points on every target graph. The smallest learning rate, $10^{-4}$, lowers the accuracy by 0.83 points on Cora and 1.84 points on CiteSeer, and changes it by at most 0.49 points on the other graphs. Hidden widths from 128 to 1{,}024 change the accuracy by at most 0.87 points, and the default width of 512 gives the best accuracy on five of the six target graphs. A width of 2{,}048 lowers the accuracy on every target graph, by 3.06 points on CiteSeer and by at most 1.51 points on the others.

\begin{figure}[htbp]
\centering
\includegraphics[width=\linewidth]{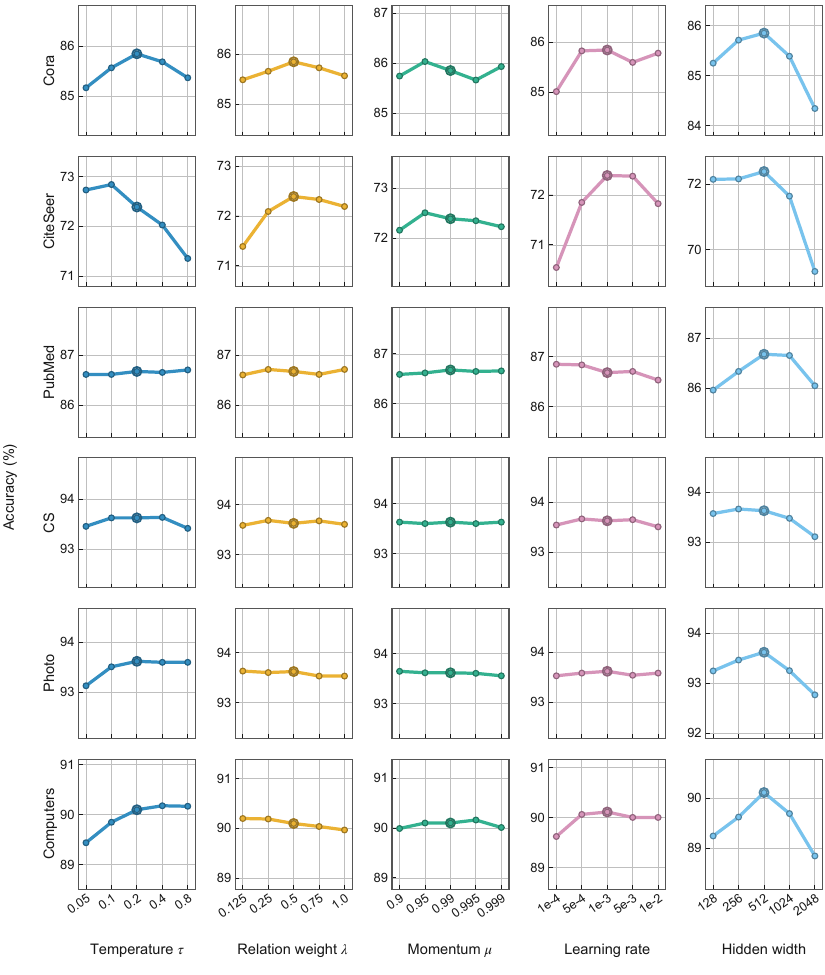}
\caption{Accuracy (\%) of node classification after 2,000 pre-training epochs when the temperature $\tau$, the relation weight $\lambda$, the momentum $\mu$, the learning rate, or the hidden width of both branches is varied. Each row is a target graph, and the circled point is the default value.}
\label{fig:sensitivity}
\end{figure}

\FloatBarrier

\end{document}